\documentclass[12pt,journal,onecolumn]{IEEEtran}
\usepackage{amsmath,amsfonts,amssymb}
\usepackage{algorithmic}
\usepackage[linesnumbered,ruled]{algorithm2e}
\usepackage[sort&compress, numbers]{natbib}
\usepackage{array}
\usepackage{textcomp}
\usepackage{stfloats}
\usepackage{url}
\usepackage{verbatim}
\usepackage{graphicx}
\usepackage{amsthm}
\usepackage{optidef}
\usepackage{subcaption}
\usepackage{xcolor}
\allowdisplaybreaks

\DeclareMathOperator{\Clip}{Clip}

\DeclareMathOperator{\Distest}{DistEst}

\DeclareMathOperator{\Est}{Estimate}
\newtheorem{thm}{Theorem}
\newtheorem{ass}{Assumption}
\newtheorem{defi}{Definition}
\newtheorem{lem}{Lemma}

\newcommand{\norm}[1]{\left\lVert#1\right\rVert}
\begin{document}

\title{Differential Private Stochastic Optimization with Heavy-tailed Data: Towards Optimal Rates}

\author{
	Puning Zhao$^1$ \quad Jiafei Wu$^1$ \quad Zhe Liu$^1$ \quad Chong Wang$^2$ \quad Rongfei Fan$^3$ \quad Qingming Li$^4$\\
\small $^1$ Zhejiang Lab \quad $^2$ Ningbo University \quad $^3$ Beijing Institute of Technology \quad $^4$ Zhejiang University\\
\texttt{\{pnzhao,wujiafei,zhe.liu\}@zhejianglab.com},\\
\texttt{wangchong@nbu.edu.cn}, \texttt{fanrongfei@bit.edu.cn},
\texttt{liqm@zju.edu.cn}
}



\maketitle
\begin{abstract}
	We study convex optimization problems under differential privacy (DP). With heavy-tailed gradients, existing works achieve suboptimal rates. The main obstacle is that existing gradient estimators have suboptimal tail properties, resulting in a superfluous factor of $d$ in the union bound. In this paper, we explore algorithms achieving optimal rates of DP optimization with heavy-tailed gradients. Our first method is a simple clipping approach. Under bounded $p$-th order moments of gradients, with $n$ samples, it achieves $\tilde{O}(\sqrt{d/n}+\sqrt{d}(\sqrt{d}/n\epsilon)^{1-1/p})$ population risk with $\epsilon\leq 1/\sqrt{d}$. We then propose an iterative updating method, which is more complex but achieves this rate for all $\epsilon\leq 1$. The results significantly improve over existing methods. Such improvement relies on a careful treatment of the tail behavior of gradient estimators. Our results match the minimax lower bound in \cite{kamath2022improved}, indicating that the theoretical limit of stochastic convex optimization under DP is achievable.
\end{abstract}

%
\section{Introduction}

Differential privacy (DP) \cite{dwork2006calibrating} is a prevailing framework for privacy protection. In recent years, significant progress has been made on deep learning under DP \cite{abadi2016deep,tramer2021differentially,wei2022dpis,de2022unlocking}. While the practical performance continues to improve, the theoretical analysis lags behind. Existing analyses focus primarily on Lipschitz loss functions, such that the gradients are all bounded \cite{bassily2014private,bassily2019private,iyengar2019towards,bassily2023user}. However, many empirical studies have shown that in deep learning, gradient noise usually follows heavy-tailed distributions \cite{simsekli2019tail,csimcsekli2019heavy,zhang2020adaptive,gurbuzbalaban2021heavy}. To bridge the gap between theory and practice, it is worth investigating the DP stochastic optimization problem with heavy tails.

It has been shown in \cite{kamath2022improved} that if the stochastic gradients have bounded $p$-th order moments for some $p\geq 2$, then the minimax lower bound of optimization risk is $\Omega\left(\sqrt{d/n}+\sqrt{d}\left(\sqrt{d}/\epsilon n\right)^{1-1/p}\right)$ under $(\epsilon, \delta)$-DP, which can be viewed as the theoretical limit of DP optimization. However, existing methods fail to achieve this rate. Compared with Lipschitz loss functions, a crucial challenge in analyzing heavy-tailed gradients is the design of an efficient mean estimator under DP. Various methods for DP mean estimation have been proposed \cite{huang2021instance,hopkins2022efficient,kamath2020private,liu2021robust}. These methods have achieved optimal mean squared error, but the high probability bounds are not optimal. To bound the risk of optimization, we need a union bound of the bias and variance of gradient estimate over the whole hypothesis space. Therefore, a suboptimal high probability bound of mean estimation results in a suboptimal risk of optimization. To be best of our knowledge, currently, it is unknown whether the minimax lower bound shown in \cite{kamath2022improved} is achievable.

In this paper, we answer this question affirmatively. We propose two methods, called the \emph{simple clipping} method and the \emph{iterative updating} method, respectively. For both methods, we derive the high probability bounds of mean estimation first, and then analyze the risk of optimization. 

\emph{1) Simple clipping.} This method just clips all gradients to a given radius $R$ and calculates sample averages. $R$ can be tuned based on the privacy requirement $\epsilon, \delta$ and the number of samples $n$. Our analysis shows that the population risk is $\tilde{O}\left(\sqrt{d/n}+\sqrt{d}\left(\sqrt{d}/\epsilon n\right)^{1-1/p} + d^{\frac{3}{2}-\frac{1}{p}}/n^{1-\frac{1}{p}} \right)$, which improves over existing methods. This rate matches the minimax lower bound if $\epsilon\leq 1/\sqrt{d}$, under which the third term $d^{\frac{3}{2}-\frac{1}{p}}/n^{1-\frac{1}{p}}$ does not dominate. The key of such improvement is that we treat the tail behavior of mean estimation more carefully. In particular, we show that the mean estimation has a subexponential tail in all directions, which refines the union bounds and eventually leads to the risk bound mentioned above. The remaining drawback is that this method has an additional term $d^{\frac{3}{2}-\frac{1}{p}}/n^{1-\frac{1}{p}}$. Therefore, this method is suboptimal if $\epsilon>1/\sqrt{d}$.


\emph{2) Iterative updating.} This method is proposed to remove the additional term of the simple clipping method. It divides the data into $k$ groups. For each group, this method calculates the group-wise mean and adds noise to meet DP requirements. After that, the mean estimate is iteratively updated based on the estimation of distances and directions to the ground truth $\nabla F(\mathbf{w}_t)$. Such design is inspired by several existing methods for non-private mean estimation with heavy-tailed data \cite{lugosi2019sub,cherapanamjeri2019fast,lei2020fast,depersin2022robust}. Compared with the simple clipping approach, this method improves the tail behavior of the mean estimator from subexponential to subgaussian. Moreover, this method is invariant to permutations of groups. As a result, the overall privacy of the final estimate is amplified compared with the privacy of each group \cite{erlingsson2019amplification,feldman2022hiding}. With this new algorithm and refined theoretical analysis, we achieve a risk bound $\tilde{O}\left(\sqrt{d/n}+\sqrt{d}\left(\sqrt{d}/\epsilon n\right)^{1-1/p}\right)$, matching the minimax lower bound. 

\begin{table}[t]
	\small
	\begin{center}
		\begin{tabular}{c|c}
			\hline
			Source &Bound of risk\\
			\hline
			\cite{wang2020differentially} & $\tilde{O}\left(\left(\frac{d^3}{\epsilon^2 n}\right)^\frac{1}{3} \right)$ \footnotemark\\
			
			(Kamath et al. 2022) & $\tilde{O}\left(\frac{d}{\sqrt{n}}+\sqrt{d}\left(\frac{d^{3/2}}{\epsilon n}\right)^{1-\frac{1}{p}} \right)$\\
			
			(Kamath et al. 2022) & $\tilde{O}\left(\underset{0.5\leq q\leq 2}{\min}\left(\frac{d^\frac{3-q}{2}}{\sqrt{n}}+\frac{d^\frac{1+q}{2}}{\epsilon^\frac{1}{2}\sqrt{n}}\right)\right)$ \footnote{\cite{kamath2022improved} proposed two methods}\\
			\textbf{Simple clipping} &\hspace{-2mm} $\tilde{O}\left(\sqrt{\frac{d}{n}}+\sqrt{d}\left(\frac{\sqrt{d}}{\epsilon n}\right)^{1-\frac{1}{p}}+\frac{d^{\frac{3}{2}-\frac{1}{p}}}{n^{1-\frac{1}{p}}} \right)$\hspace{-2mm} \\
			\textbf{Iterative updating} & $\tilde{O}\left(\sqrt{\frac{d}{n}}+\sqrt{d}\left(\frac{\sqrt{d}}{\epsilon n}\right)^{1-\frac{1}{p}} \right)$\\
			\hline
			Lower bound  & $\Omega\left(\sqrt{\frac{d}{n}}+\sqrt{d}\left(\frac{\sqrt{d}}{\epsilon n}\right)^{1-\frac{1}{p}} \right)$ \\
			\hline
		\end{tabular}
		\caption{Comparison of risk bounds of stochastic optimization under $(\epsilon, \delta)$-DP with $p$-th order bounded moments on gradients. Logarithmic factors are omitted here.}\label{tab}
	\end{center}
\end{table}

\footnotetext[1]{The analysis in \cite{wang2020differentially} underestimates the dependence on $d$. We refer to \cite{kamath2022improved} for a detailed discussion. The bounds listed in Table \ref{tab} are the corrected results from \cite{kamath2022improved}.}

\footnotetext[2]{\cite{kamath2022improved} proposed two methods, whose risk bounds are shown in the second and third rows in Table \ref{tab}, respectively.}

Our results and comparison with existing works are summarized in Table \ref{tab}. To the best of our knowledge, our methods achieve the optimal risk of stochastic convex optimization problems under DP for the first time.

\section{Related Work}

\textbf{DP optimization.} Early works focus on empirical risk minimization (ERM) under DP, which is a relatively simpler problem compared with stochastic optimization, such as \cite{chaudhuri2008privacy,kifer2012private,thakurta2013differentially,bassily2014private,wang2017differentially,zhang2021wide}. For stochastic optimization problem, \cite{bassily2019private} shows that for Lipschitz loss functions, DP-SGD is minimax optimal with proper parameter selection. The analysis is then improved in several later works on time complexity \cite{feldman2020private,kulkarni2021private} and extended to different geometries \cite{asi2021private,bassily2021non}. For heavy-tailed gradients, the non-private optimization has been widely studied \cite{pascanu2012understanding,zhang2020adaptive,gorbunov2020stochastic,parletta2022high,liu2023high,nguyen2023improved,eldowa2024general,liu2023stochastic,armacki2023high,koloskova2023revisiting}. Under DP, \cite{wang2020differentially} provides the upper bound of DP-SGD under the assumption that gradients have bounded second moments. \cite{hu2022high} analyzes the sparse setting. \cite{kamath2022improved} improves on the risk bound. Moreover, \cite{das2023beyond} weakens the uniform Lipschitz assumption to a sample-wise one.

\textbf{Mean estimation with subgaussian rates.} Non-private mean estimation for heavy-tailed distributions has received widespread attention \cite{lugosi2019mean}. We hope to minimize the $1-\beta$ high probability bound of the estimation error. \cite{minsker2015geometric} shows that the median-of-means method achieves an error bound of $O(\sqrt{d\ln(1/\beta)/n})$ with probability $1-\beta$. \cite{lugosi2019sub} improves the bound to $O(\sqrt{(d+\ln(1/\beta))/n})$ for the first time, but the method in \cite{lugosi2019sub} is computationally expensive. After that, improved algorithms with the same high probability bounds and faster computation are proposed in \cite{cherapanamjeri2019fast,lei2020fast,depersin2022robust}. Note that this rate is minimax optimal. \cite{catoni2012challenging} shows that the lower bound of estimation error with $1-\beta$ probability is $\Omega(\sqrt{(d+\ln(1/\beta))/n})$ for $n$ samples.

Compared with non-private mean estimation, we need to randomize samples carefully to achieve a tradeoff between accuracy and privacy. This involves a refined analysis of tail behaviors, as well as privacy amplification by shuffling. As a result, we finally achieve optimal rates of DP optimization with heavy-tailed gradients.

\section{Preliminaries}
Denote $\mathcal{Z}$ as the space of samples, and $\mathcal{Y}$ as the output space. We state the standard definition of DP first.
\begin{defi}
	(Differential Privacy (DP) \cite{dwork2006calibrating}) A randomized algorithm $\mathcal{A}:\mathcal{Z}^n\rightarrow \mathcal{Y}$ satisfies $(\epsilon, \delta)$-DP if for any $S\subseteq \mathcal{Y}$ and any pairs of datasets $D,D'\in \mathcal{Z}^n$ such that $D$ and $D'$ differ in one element,
	\begin{eqnarray}
		\text{P}(\mathcal{A}(D)\in S)\leq e^\epsilon\text{P}(\mathcal{A}(D')\in S)+\delta.
		\label{eq:dp}
	\end{eqnarray}
	Moreover, $\mathcal{A}$ is $\epsilon$-DP if \eqref{eq:dp} holds with $\delta = 0$.
\end{defi}

For the convenience of analysis, we also introduce another definition of DP, called concentrated differential privacy, which was first proposed in \cite{dwork2016concentrated}. \cite{bun2016concentrated} gives a refinement called zero-concentrated differential privacy. Throughout this paper, we use the definition in \cite{bun2016concentrated}.

\begin{defi}
	(Concentrated differential privacy (CDP) \cite{bun2016concentrated}) A randomized algorithm $\mathcal{A}:\mathcal{Z}^n\rightarrow \mathcal{Y}$ satisfies $\rho$-CDP if for any pairs of datasets $D,D'\in \mathcal{Z}^n$ such that $D$ and $D'$ differ in one element, any $S\subseteq \mathcal{Y}$, and any $\alpha\in (1,\infty)$,
	$D_\alpha(\mathcal{A}(D)||\mathcal{A}(D'))\leq \rho\alpha$,
	in which $D_\alpha$ is the $\alpha$-R{\'e}nyi divergence between two random variables.\footnote{The $\alpha$-R{\'e}nyi divergence between two distributions $P$ and $Q$ is defined as
		$D_\alpha(P||Q)=\frac{1}{\alpha-1}\ln \mathbb{E}_{X\sim Q}\left[\left(\frac{P(X)}{Q(X)}\right)^\alpha \right]$.}
\end{defi}

Our analysis in this paper will use some basic rules about the composition of DP and CDP, as well as the conversion between them. These rules are summarized in Lemma \ref{lem:basics}.

\begin{lem}\label{lem:basics}
	There are several facts about DP and CDP:
	
	(1) (Advanced composition, \cite{dwork2010boosting,dwork2014algorithmic}) If $\mathcal{A}_1,\ldots, \mathcal{A}_k$ are $(\epsilon, \delta)$-DP, then the composition $(\mathcal{A}_1,\ldots, \mathcal{A}_k)$ is $(\sqrt{2k\ln(1/\delta')}\epsilon+k\epsilon(e^\epsilon-1), k\delta+\delta')$-DP for any $\delta'\in (0,1)$;
	
	(2) (Composition of CDP, \cite{bun2016concentrated}) If $\mathcal{A}_1,\ldots, \mathcal{A}_k$ are $\rho$-CDP, then the composition $(\mathcal{A}_1,\ldots, \mathcal{A}_k)$ is $k\rho$-CDP;
	
	(3) (From DP to CDP, \cite{bun2016concentrated}) If a randomized algorithm $\mathcal{A}:\mathcal{Z}^n\rightarrow \mathcal{Y}$ is $\epsilon$-DP, then $\mathcal{A}$ is $(\epsilon^2/2)$-CDP;
	
	(4) (From CDP to DP, \cite{bun2016concentrated}) If $\mathcal{A}$ is $(\epsilon^2/2)$-CDP, then $\mathcal{A}$ is $(\epsilon^2/2+\epsilon\sqrt{2\ln(1/\delta)}, \delta)$-DP.
\end{lem}

Moreover, we need the following lemma on the noise mechanism.
\begin{lem}\label{lem:gaussian}
	(Additive noise mechanism, \cite{bun2016concentrated}) Let $\mathcal{A}_0$ be a non-private algorithm. Define $$\Delta_2(\mathcal{A}_0)=\max_{d_H(D,D')=1}\norm{\mathcal{A}_0(D)-\mathcal{A}_0(D')}_2$$
	as the $\ell_2$ sensitivity of $\mathcal{A}_0$, in which $d_H$ denotes the Hamming distance. Then $\mathcal{A}(D)=\mathcal{A}_0(D)+\mathbf{W}$ with $\mathbf{W}\sim \mathcal{N}(0, (\Delta_2^2(\mathcal{A}_0)/2\rho)\mathbf{I})$ satisfies $\rho$-CDP.
\end{lem}

We then state the problem of stochastic optimization. Suppose there are $n$ i.i.d samples $\mathbf{Z}_1,\ldots, \mathbf{Z}_n$ following a common distribution. Given a convex constraint $\mathcal{W}\subseteq \mathbb{R}^d$ and loss function $l:\mathcal{W}\times \mathcal{Z} \rightarrow \mathbb{R}$ which is convex in $\mathcal{W}$, the goal is to find an estimated minimizer $\hat{\mathbf{w}}$ of the population risk
\begin{eqnarray}
	F(\mathbf{w}) := \mathbb{E}[l(\mathbf{w}, \mathbf{Z})].
\end{eqnarray}
Denote
$\mathbf{w}^*={\arg\min}_\mathbf{w} F(\mathbf{w})$
as the minimizer of the population risk. The performance of a learning algorithm is evaluated by the expected excess risk
$\mathbb{E}[F(\hat{\mathbf{w}})] - F(\mathbf{w}^*)$.
Our analysis is based on the following assumptions, which are similar to \cite{kamath2022improved}, with simplified statements.
\begin{ass}\label{ass}
	There exists constants $L$, $\lambda$, $M$ such that
	
	(a) The diameter of parameter space $\mathcal{W}$ is bounded by $L$;
	
	(b) $F$ is $\lambda$-smooth, i.e. for any $\mathbf{w}, \mathbf{w}'$,
	\begin{eqnarray}
		F(\mathbf{w}')\leq F(\mathbf{w})+\langle \nabla F(\mathbf{w}), \mathbf{w}'-\mathbf{w}\rangle + \frac{\lambda}{2}\norm{\mathbf{w}-\mathbf{w}'}^2;	
	\end{eqnarray}
	
	(c) The gradients of loss function has $p$-th order bounded moment for some $p\geq 2$. To be more precise, for any $\mathbf{w}\in \mathcal{W}$ and any vector $\mathbf{u}$ with $\norm{\mathbf{u}}=1$, 
	\begin{eqnarray}
		\mathbb{E}\left[|\langle \mathbf{u}, \nabla l(\mathbf{w}, \mathbf{Z})\rangle|^p\right] \leq M^p.
		\label{eq:tail}
	\end{eqnarray}
	
	In (b) and (c), $\norm{\cdot}$ denotes $\ell_2$ norm.
\end{ass}

In Assumption \ref{ass}, (a) and (b) are common in literatures about convex optimization. (c) controls the tail behavior of gradient vectors. Lower $p$ indicates a heavier tail, and vice versa. The case with the Lipschitz loss function (i.e. bounded gradients) corresponds to the limit of $p\rightarrow \infty$. Our assumption \eqref{eq:tail} is slightly different from the assumptions in \cite{kamath2022improved}. In \cite{kamath2022improved}, it is required that for all $\mathbf{w}\in \mathcal{W}$ and all $j\in [d]$,
\begin{eqnarray}
	\mathbb{E}\left[|\langle \mathbf{e}_j, \nabla l(\mathbf{w}, \mathbf{Z})-\nabla F(\mathbf{w})\rangle|^p\right]\leq M^p,
	\label{eq:tailb}
\end{eqnarray}
in which $\mathbf{e}_j$, $j=1,\ldots,d$ form an orthonormal basis. In \eqref{eq:tail}, to make the assumption more natural, we impose moment bounds on every unit vector $\mathbf{u}$, instead of only on basis vectors. Another difference is that in \eqref{eq:tail}, the moment bound is imposed directly on $\nabla l(\mathbf{w}, Z)$ instead of the deviation from its mean $\nabla F(\mathbf{w})$. Since \cite{kamath2022improved} also requires that $\norm{\nabla F(\mathbf{w})}$ is bounded by $O(1)$ (see \cite{kamath2022improved}, Assumption 2.12 (6)), we do not introduce additional restriction compared with \cite{kamath2022improved} in this aspect.

Under Assumption \ref{ass}, the minimax lower bound of optimization under DP has been established in \cite{kamath2022improved}. For consistency of notations, we restate it in the following theorem.

\begin{thm}\label{thm:lower}
	(Rephrased from \cite{kamath2022improved}, Theorem 6.4) Let $\mathcal{F}$ be the set of all $\lambda$-smooth functions on $\mathcal{W}$. Let $\hat{\mathbf{w}}=\mathcal{A}(\mathbf{Z}_1,\ldots, \mathbf{Z}_n)$, in which $\mathcal{A}:\mathcal{Z}^n \rightarrow \mathcal{W}$ is an arbitrary learning algorithm satisfying $(\epsilon, \delta)$-DP. Then
	\begin{eqnarray}
		\inf_\mathcal{A}\sup_{F\in \mathcal{F}} \mathcal{R}(\hat{\mathbf{w}}) \gtrsim \sqrt{\frac{d}{n}}+\sqrt{d}\left(\frac{\sqrt{d}}{n\epsilon}\right)^\frac{p-1}{p}\ln \frac{1}{\delta}.
		\label{eq:mmx}
	\end{eqnarray}
\end{thm}

Theorem \ref{thm:lower} describes the theoretical limit of optimization risk under the differential privacy requirements. Under the light tail limit, i.e. $p\rightarrow \infty$, the right hand side of \eqref{eq:mmx} becomes $\Omega(\sqrt{d/n}+d/(n\epsilon))$. Recall that for bounded gradients, the bound of excess risk is $O(1/\sqrt{n})+\sqrt{d}/(n\epsilon))$ \cite{bassily2019private}. At first glance, it appears that the result in Theorem \ref{thm:lower} is larger by a factor of $\sqrt{d}$. However, this discrepancy comes from the difference of assumptions. Under our tail assumption \eqref{eq:tail}, the expectation of the $\ell_2$ norm of the gradient vector is only bounded by $O(\sqrt{d})$, while \cite{bassily2019private} requires the gradients to be bounded by $O(1)$. After adjustments of assumptions, Theorem \ref{thm:lower} matches \cite{bassily2019private} under the limit $p\rightarrow \infty$. Similar discussions can also be found in \cite{kamath2022improved}.

Finally, we discuss the convergence property of stochastic optimization. The framework is shown as follows. At each step $t$, let $g(\mathbf{w}_t)$ be the gradient estimate by $\nabla l(\mathbf{w}_t, \mathbf{Z}_1), \ldots, \nabla l(\mathbf{w}_t, \mathbf{Z}_n)$ using either Algorithm \ref{alg:est} or \ref{alg:est_improve} with some appropriate privacy constraints. The model weights are then updated with
\begin{eqnarray}
	\mathbf{w}_{t+1}=\Pi_\mathcal{W}(\mathbf{w}_t-\eta g(\mathbf{w}_t)),
	\label{eq:update}
\end{eqnarray}
in which $\Pi_\mathcal{W}$ is the projection operator on $\mathcal{W}$. Finally, the algorithm returns
\begin{eqnarray}
	\hat{\mathbf{w}}=\frac{1}{T}\sum_{t=1}^T \mathbf{w}_t.
	\label{eq:output}
\end{eqnarray}

\begin{algorithm}[t]
	\caption{Stochastic optimization}\label{alg:so}
	\begin{algorithmic}
		\STATE \textbf{Input:} dataset $\{\mathbf{Z}_1,\ldots, \mathbf{Z}_n \}$, privacy requirement $(\epsilon, \delta)$
		\STATE \textbf{Output:} Final iterate $\hat{\mathbf{w}}$
		\STATE \textbf{Parameter:} Initial point $\mathbf{w}_0$, learning rate $\eta$, number of steps $T$
		\FOR{$t=1,\ldots, T$}
		\STATE Calculate $g(\mathbf{w}_t)$, which estimates $\nabla F(\mathbf{w}_t)$ using $\nabla l(\mathbf{w}_t,\mathbf{Z}_1),\ldots, \nabla l(\mathbf{w}_t,\mathbf{Z}_n)$;
		\STATE $\mathbf{w}_{t+1}=\Pi_\mathcal{W} (\mathbf{w}_t-\eta g(\mathbf{w}_t))$;
		\ENDFOR
		\STATE $\hat{\mathbf{w}}=(1/T)\sum_{t=1}^T \mathbf{w}_t$;
		\RETURN $\hat{\mathbf{w}}$
	\end{algorithmic}
\end{algorithm}

The whole procedures are shown in Algorithm \ref{alg:so}. The risk can be bounded using the bias and variance of gradient estimates.
\begin{lem}\label{lem:opt}
	(\cite{kamath2022improved}, Theorem 3.1) Define
	\begin{eqnarray}
		B:=\max_t \norm{\mathbb{E}[g(\mathbf{w}_t)]-\nabla F(\mathbf{w}_t)},
		\label{eq:bdef}
	\end{eqnarray}
	\begin{eqnarray}
		G^2:=\max_t\mathbb{E}[\norm{g(\mathbf{w}_t)-\nabla F(\mathbf{w}_t)}^2].
		\label{eq:gdef}
	\end{eqnarray}
	Then the risk of optimization with updating rule \eqref{eq:update} and output \eqref{eq:output} is bounded by
	\begin{eqnarray}
		\mathbb{E}[F(\hat{\mathbf{w}})]- F(\mathbf{w}^*)\leq \frac{L^2}{2\eta T}+LB+\eta(\lambda^2L^2+G^2).
	\end{eqnarray}
\end{lem}
For completeness, we show the proof of Lemma \ref{lem:opt} in Appendix A. Based on Lemma \ref{lem:opt}, to bound the excess risk, we need to give bounds of $B$ and $G^2$. It is relatively simple to bound $\norm{\mathbb{E}[g(\mathbf{w})]-\nabla F(\mathbf{w})}$ and $\mathbb{E}[\norm{g(\mathbf{w})-\nabla F(\mathbf{w})}^2]$ for any fixed $\mathbf{w}$. The challenging part is that $\mathbf{w}_t$ depends on the data, therefore the bounds with respect to fixed $\mathbf{w}$ do not imply the bounds of $B$ and $G^2$. In the following two sections, we propose two methods and provide bounds of $B$ and $G^2$ for each method.

\section{Simple Clipping Method}
The simple clipping method is stated as follows. In each round, the algorithm just clips the gradient to some radius $R$ and then adds noise to protect the privacy. Such a simple clipping method is convenient to implement and is close to the popular DP-SGD algorithm in \cite{abadi2016deep}. Therefore, an in-depth analysis of this method will be helpful to bridge the gap between theory and practice in deep learning with DP.

\subsection{Mean Estimation}\label{sec:mean}
Suppose there are $n$ i.i.d samples $\mathbf{X}_1,\ldots,\mathbf{X}_n$ following a common distribution with mean $\mu$. Here we assume that for any unit vector $\mathbf{u}$ with $\norm{\mathbf{u}}=1$, $\mathbb{E}[|\langle \mathbf{u}, \mathbf{X}\rangle|^p]\leq M^p$, which matches Assumption \ref{ass}(c).

Since samples follow a heavy-tailed distribution, some of them may be far away from $\mu$. A simple averaging of these samples has infinite sensitivity. To ensure that the overall sensitivity is bounded, we clip them with a radius $R$. To be more precise, for each $i=1,\ldots, n$, let
\begin{eqnarray}
	\mathbf{Y}_i=\Clip(\mathbf{X}_i, R),
	\label{eq:yi}
\end{eqnarray}	
in which
$\Clip(\mathbf{x}, R)=\min\left\{1, R/\norm{\mathbf{x}} \right\}\mathbf{x}$.
Then the final estimate is
\begin{eqnarray}
	\hat{\mu}=\frac{1}{n}\sum_{i=1}^n \mathbf{Y}_i+\mathbf{W}
	\label{eq:muhat},
\end{eqnarray}
in which $\mathbf{W}\sim \mathcal{N}(0, \sigma^2 \mathbf{I})$ is the noise added to meet the privacy requirement. We then show the following theorem, which determines the strength of $\mathbf{W}$, 
\begin{lem}\label{lem:var}
	Let
	$\sigma^2=2R^2/(\rho n^2)$,
	then $\hat{\mu}$ is $\rho$-CDP.
\end{lem}
\begin{proof}
	Since $\norm{\mathbf{Y}_i}\leq R$, the sensitivity of $(1/n)\sum_{i=1}^n \mathbf{Y}_i$ is $\Delta_2(\hat{\mu})=2R/n$. According to Lemma \ref{lem:gaussian}, the estimator \eqref{eq:muhat} is $\rho$-CDP.
\end{proof}

The procedure is summarized in Algorithm \ref{alg:est}.
\begin{algorithm}
	\caption{Simple clipping method for mean estimation}\label{alg:est}
	\textbf{Input:} dataset $\{\mathbf{X}_1,\ldots, \mathbf{X}_n\}$ and privacy requirement $\rho$ (under CDP)\\
	\textbf{Output:} Estimate $\hat{\mu}$\\
	\textbf{Parameter:} $R$
	\begin{algorithmic}[1]
		\FOR{$i=1,\ldots, n$}
		\STATE $\mathbf{Y}_i=\Clip(\mathbf{X}_i, R)$;
		\ENDFOR
		\STATE $\hat{\mu}=(1/n)\sum_{i=1}^n \mathbf{Y}_i + \mathbf{W}$, in which $\mathbf{W}\sim \mathcal{N}(0,2R^2/(\rho n^2))$;
		\RETURN $\hat{\mu}$
	\end{algorithmic}
\end{algorithm}

The following theorem provides a high probability bound of the estimation error.

\begin{thm}\label{thm:meanest}
	Under the condition $\mathbb{E}[|\langle \mathbf{u}, \mathbf{X}\rangle|^p]\leq M^p$ for some $p\geq 2$, under $\rho$-CDP, with probability $1-\beta$, the simple clipping method achieves
	\begin{eqnarray}
		\norm{\hat{\mu}-\mu}&\leq& \max\left\{\sqrt{\frac{12M^2}{n}\ln \frac{2\times 6^d}{\beta}}, \frac{8R}{n} \ln \frac{2\times 6^d}{\beta} \right\}+\frac{d^\frac{p}{2}M^p}{p-1}R^{1-p}+\frac{4R}{n\sqrt{\rho}}\sqrt{\ln \frac{2\times 6^d}{\beta}}.
		\label{eq:errsimple}
	\end{eqnarray}
\end{thm} 

Now we provide an intuitive interpretation of the result. The second term $d^{p/2}M^pR^{1-p}/(p-1)$ in \eqref{eq:errsimple} is the clipping bias $\norm{\mu_Y-\mu}$, in which $\mu_Y=\mathbb{E}[\mathbf{Y}_i]$ is the expectation after clipping. The third term in \eqref{eq:errsimple} is caused by the noise $\mathbf{W}$. The first term is the bound of $\norm{\bar{\mathbf{Y}}-\mu_Y}$, which is caused by the randomness of samples. Here $\bar{\mathbf{Y}}=(1/n)\sum_{i=1}^n \mathbf{Y}_i$ is the sample average of $\mathbf{Y}_i$. It can be written as $O\left(\sqrt{\frac{d+\ln(1/\beta)}{n}}+\frac{R}{n}(d+\ln\frac{1}{\beta})\right)$, indicating that $\bar{\mathbf{Y}}$ is subgaussian around its mean $\mu_Y$, followed by a subexponential tail.

In \eqref{eq:errsimple}, the factor $\ln(2\times 6^d/\beta)$ is an important improvement over \cite{kamath2022improved}. The corresponding factor in \cite{kamath2022improved} is $O(d\ln (1/\beta))$, while we achieve $O(d+\ln (1/\beta))$ here. Such difference does not lead to improvement in the bias and variance of mean estimation. However, the high probability bound is improved significantly. In optimization problems, we need to take union bounds over all possible model weights $\mathbf{w}$, which requires $\beta$ to be very small. In this case, $d+\ln(1/\beta)\ll d\ln(1/\beta)$. As a result, our method improves over \cite{kamath2022improved} in the dependence of $d$. Despite such improvement, \eqref{eq:errsimple} has a drawback of exponential tail. As will be shown later, due to the subexponential tail $\frac{8R}{n}\ln \frac{2\times 6^d}{\beta}$, the optimization risk is not completely optimal.

\subsection{Optimization}\label{sec:optim}
Based on the simple clipping approach, we then analyze the performance of stochastic optimization. We first discuss the DP property of the optimization problem.

\begin{thm}\label{thm:rho}
	If $\epsilon\leq 1$, let the gradient estimator be the simple clipping method under $\rho/T$-CDP, in which 
	\begin{eqnarray}
		\rho = \frac{\epsilon^2}{\left(1+2\sqrt{\ln \frac{1}{\delta}}\right)^2},
		\label{eq:rhosc}
	\end{eqnarray}
	then the whole optimization process is $(\epsilon, \delta)$-DP.
\end{thm}
\begin{proof}
	By Lemma \ref{lem:basics}(2), since each step is $\rho/T$-CDP, the whole process is $\rho$-CDP. By Lemma \ref{lem:basics}(4), $\rho$-CDP implies $(\rho+2\sqrt{\ln(1/\delta)}, \delta)$-DP. Since $\epsilon\leq 1$, from \eqref{eq:rhosc}, $\rho\leq 1$. Therefore
	\begin{eqnarray}
		\rho + 2\sqrt{\rho \ln \frac{1}{\delta}}\leq \sqrt{\rho}\left(1+2\sqrt{\ln \frac{1}{\delta}}\right) \leq \epsilon.
	\end{eqnarray}
	Therefore the optimization process is $(\epsilon, \delta)$-DP.
\end{proof}

From Theorem \ref{thm:rho}, we let each step satisfy $\rho/T$-CDP, in which $\rho$ takes value according to \eqref{eq:rhosc}. By Lemma \ref{lem:var}, this requires the noise variance be $\sigma^2=2R^2T/(\rho n^2)$. As discussed earlier, $\mathbf{w}_t$ depends on previous steps, which depend on the data. Therefore, we need to get union bounds of estimation error to calculate $B$ and $G$ defined in \eqref{eq:bdef} and \eqref{eq:gdef}. The results are shown in the following lemma.

\begin{lem}\label{lem:bg}
	$B$ and $G^2$ defined in \eqref{eq:bdef} and \eqref{eq:gdef} are bounded by
	\begin{eqnarray}
		B\lesssim \sqrt{\frac{d\ln n}{n}}+\frac{Rd\ln n}{n}\ln n+d^\frac{p}{2}R^{1-p},
		\label{eq:bbound}
	\end{eqnarray}
	\begin{eqnarray}
		G^2\lesssim \frac{d\ln n}{n} +\frac{R^2d^2}{n^2}\ln^2 n + d^p R^{2(1-p)}+\frac{R^2Td}{\rho n^2}.
		\label{eq:gbound}
	\end{eqnarray}
\end{lem}
With Lemma \ref{lem:opt} and \ref{lem:bg}, we then show the following theorem, which bounds the overall excess risk.
\begin{thm}\label{thm:risk}
	Let $T=\rho n^2/(dR^2)$, $\eta=1/\sqrt{2T\lambda^2}$, and
	\begin{eqnarray}
		R=\sqrt{d}\left(\frac{n\sqrt{\rho}}{\sqrt{d}}\right)^\frac{1}{p}\wedge \sqrt{d}\left(\frac{n}{d}\right)^\frac{1}{p},
	\end{eqnarray}
	in which $\rho$ is determined with \eqref{eq:rhosc}. Then under Assumption \ref{ass}, the excess risk of Algorithm \ref{alg:so} is bounded by
	\begin{eqnarray}
		\mathbb{E}[F(\hat{\mathbf{w}})] - F(\mathbf{w}^*)\lesssim \sqrt{\frac{d\ln n}{n}}+\sqrt{d}\left(\frac{\sqrt{d}}{n\epsilon}\sqrt{\ln \frac{1}{\delta}}\right)^{1-\frac{1}{p}}+\frac{d^{\frac{3}{2}-\frac{1}{p}}}{n^{1-\frac{1}{p}}}\ln n.
		\label{eq:risk}
	\end{eqnarray}
\end{thm}

Compared with the lower bound in Theorem \ref{thm:lower}, the first two terms in \eqref{eq:risk} match \eqref{eq:mmx} up to logarithmic factors. However, there is an additional term $d^{3/2-1/p}\ln n/n^{1-1/p}$. If $\epsilon\leq 1/\sqrt{d}$, then this term does not dominate. Therefore, the simple clipping method is minimax optimal (up to logarithmic factors) for $\epsilon\leq 1/\sqrt{d}$.

\section{Iterative Updating Method}
The previous section shows that the simple clipping method is not always optimal due to an additional term $d^{3/2-1/p}\ln n/n^{1-1/p}$. In this section, we show an improved method to avoid this term, which is inspired by some existing methods for non-private mean estimation \cite{cherapanamjeri2019fast,lugosi2019sub,lei2020fast}. To begin with, to illustrate the idea of design, we provide some basic intuition. The mean estimation algorithm is then described in detail. Finally, we analyze the risk of optimization with the new mean estimator.
\subsection{Intuition}\label{sec:intuition}

The suboptimality of the simple clipping approach comes from the subexponential tails. Ideally, under $\rho$-CDP, we would like a $\norm{\hat{\mu}-\mu}\lesssim \sqrt{(d+\ln(1/\beta))/n} + R\sqrt{d+\ln(1/\beta)}/(n\sqrt{\rho})+d^{p/2}R^{1-p}$ error bound that holds with probability $1-\beta$. However, from \eqref{eq:errsimple}, the simple clipping method has an additional term $O(R(d+\ln(1/\beta))/n)$, which indicates a subexponential tail behavior. To remove the subexponential tail, a classical approach is median-of-means, which divides data into multiple groups, calculates the mean of each group, and then finds the median of all group-wise means. However, \cite{minsker2015geometric} shows that even for non-private estimation, the geometric median-of-mean method achieves a suboptimal bound of $O(\sqrt{d\ln(1/\beta)/n})$ with probability $1-\beta$. While this bound has optimal dependence on $d$ and $n$, the dependence on $\beta$ is not optimal. The calculation of the union bound of estimation error usually encounters very small $\beta$. As a result, the suboptimal dependence of the error bound on $\beta$ leads to a larger union bound.

To handle this issue, we design a new estimator, which is inspired by several later works \cite{lugosi2019sub,cherapanamjeri2019fast,lei2020fast} that improve the non-private bound to $O(\sqrt{(d+\ln(1/\beta))/n})$. The basic idea of the mean estimator is to iteratively update the current estimate $\mathbf{c}_t$ based on the estimation of distance and direction to the truth. To make the estimator satisfy the DP requirement, we add appropriate noise. The estimator is permutation invariant with respect to the group-wise means, thus equivalently, we can view these group-wise means as being shuffled. The shuffling operation makes an amplification to DP \cite{erlingsson2019amplification,feldman2022hiding}. Therefore, each group only needs to satisfy a weaker privacy requirement than $(\epsilon, \delta)$-DP. 

\subsection{The Mean Estimation Algorithm}\label{sec:alg}

Here we state the result first and then show the construction of the mean estimator.
\begin{thm}\label{thm:mean}
	There exists a constant $c$, if $\epsilon\leq c\sqrt{(1/k)\ln(1/\delta)}$ and $\delta\in (0,1)$, there exists an estimator satisfying $(\epsilon, \delta)$-DP, such that with probability $1-\beta$,
	\begin{eqnarray}
		&&\norm{\hat{\mu}-\mu}\lesssim \sqrt{\frac{d+\ln\frac{1}{\beta}}{n}}+ d^\frac{p}{2}R^{1-p}+\frac{R}{n\epsilon} \sqrt{d+\ln\frac{1}{\beta}}\left(\ln \frac{1}{\delta}+\sqrt{\ln \frac{1}{\delta}\ln\ln \frac{1}{\beta}}\right) .
		\label{eq:mean}
	\end{eqnarray}
\end{thm}

With $\epsilon\rightarrow \infty$ or $\delta\rightarrow 1$, one can just let $R$ to be sufficiently large, then $\norm{\hat{\mu}-\mu}\lesssim \sqrt{(d+\ln(1/\beta))/n}$, which matches existing results on non-private mean estimation \cite{lugosi2019sub,cherapanamjeri2019fast,lei2020fast}. Note that the factor $d+\ln(1/\beta)$ is important. If we use the median-of-means method instead, then this factor will become $d\ln(1/\beta)$, which will yield a suboptimal union bound.

The remainder of this section explains the algorithm and proves Theorem \ref{thm:mean}. The whole process of mean estimation is shown in Algorithm \ref{alg:est_improve}. The idea uses \cite{cherapanamjeri2019fast}. The difference from \cite{cherapanamjeri2019fast} is that we need to let the result satisfy $(\epsilon, \delta)$-DP, thus the truncation radius needs to be carefully tuned. Moreover, the distance estimation is with respect to the truncated mean $\mu_Y=\mathbb{E}[\mathbf{Y}_i]$ instead of ground truth $\mu$. Compared with \cite{cherapanamjeri2019fast}, we make a simplified algorithm statement here.

\begin{algorithm}[h!]
	\caption{Iterative updating method for mean estimation}\label{alg:est_improve}
	\textbf{Input:} dataset $\{\mathbf{X}_1,\ldots, \mathbf{X}_n\}$ and privacy requirement $\epsilon$, $\delta$\\
	\textbf{Output:} Estimate $\hat{\mu}$ \\
	\textbf{Parameter:} $R$, $k$, initial point $\mathbf{c}_1$	
	\begin{algorithmic}[1]
		\STATE Divide samples into $k$ groups randomly;
		\FOR{$j=1,\ldots, k$}
		\STATE $\mathbf{Q}_j=(1/m)\sum_{i\in B_j} \mathbf{Y}_i+\mathbf{W}_j$, in which $\mathbf{W}_j\sim \mathcal{N}(0,\sigma^2)$, with $\sigma^2$ determined in \eqref{eq:var2} \label{step:avg};
		\ENDFOR
		\FOR{$l=1,\ldots, t_c$}
		\STATE $d_l, \mathbf{g}_l=\Est(\mathbf{Q}_1,\ldots, \mathbf{Q}_k, \mathbf{c}_l)$;
		\label{step:est}
		\STATE $\mathbf{c}_{l+1}=\mathbf{c}_l+\eta d_l\mathbf{g}_l$;
		\ENDFOR
		\STATE $l^*=\arg\min_l d_l$;
		\STATE $\hat{\mu}=\mathbf{c}_{l^*}$;
		\RETURN $\hat{\mu}$
	\end{algorithmic}
\end{algorithm} 

Now we explain key steps in Algorithm \ref{alg:est_improve}.

\emph{1) Group-wise averages (step \ref{step:avg}).} The algorithm begins with dividing samples into  into $k$ bins $B_1,\ldots, B_k$. The size of each bin is $m$, then $n=mk$. Let
\begin{eqnarray}
	\mathbf{Q}_j=\frac{1}{m}\sum_{i\in B_j} \mathbf{Y}_i+\mathbf{W}_j,
\end{eqnarray}
with $\mathbf{W}_j\sim \mathcal{N}(0, \sigma^2)$. Here we let
\begin{eqnarray}
	\sigma^2=\frac{2R^2}{\rho m^2}.
	\label{eq:var2}
\end{eqnarray}

The final estimate is based on $\mathbf{Q}_1,\ldots, \mathbf{Q}_k$. Note that the sensitivities of $\mathbf{Q}_j$, $j=1,\ldots, k$ are all $2R/m$ over $\mathbf{X}_i$, $i\in B_j$. By Lemma \ref{lem:gaussian}, $\mathbf{Q}_1, \ldots, \mathbf{Q}_k$ are all $\rho$-CDP. Before constructing the final estimator $\hat{\mu}$, we first show that $\hat{\mu}$ is $(\epsilon, \delta)$-DP under some necessary conditions. 
\begin{thm}\label{thm:dp}
	Suppose 
	$\epsilon\leq 8e^2 \sqrt{(1/k)\ln (4/\delta)}$.
	Let the noise variance $\sigma^2$ be determined according to \eqref{eq:var2}. If an estimator $\hat{\mu}$ only depends on $\mathbf{Q}_1,\ldots, \mathbf{Q}_k$, and is permutation invariant with respect to $\mathbf{Q}_1,\ldots, \mathbf{Q}_k$. If $\mathbf{Q}_1, \ldots, \mathbf{Q}_k$ are all $\rho$-CDP with
	\begin{eqnarray}
		\rho = \frac{1}{64e^4}\frac{\epsilon^2 k}{\ln \frac{8}{\delta}\left(1+2\sqrt{\ln \frac{12k}{\delta}}\right)^2},
		\label{eq:rho}
	\end{eqnarray}
	then $\hat{\mu}$ is $(\epsilon, \delta)$-DP.
\end{thm}
Since $\mathbf{Q}_j$ is $\rho$-CDP with respect to $\{\mathbf{X}_i|i\in B_j\}$, it is straightforward to see that any estimator $\hat{\mu}$ that depends only on $\mathbf{Q}_1,\ldots, \mathbf{Q}_k$ is also $\rho$-CDP. However, if $\hat{\mu}$ is permutation invariant with respect to $\mathbf{Q}_1,\ldots, \mathbf{Q}_k$, then the privacy guarantee becomes stronger, since $\hat{\mu}$ does not change if we shuffle $\mathbf{Q}_1,\ldots, \mathbf{Q}_k$ randomly. According to \cite{erlingsson2019amplification,feldman2022hiding}, the privacy guarantee can be amplified. As a result, to ensure that the whole algorithm satisfies $(\epsilon, \delta)$-DP, the privacy requirement for each group is only $\rho$-CDP with \eqref{eq:rho}, which is weaker than \eqref{eq:rhosc} for sufficiently large $k$.

Under such settings, we show the bound of the estimation error. Similar to existing research on non-private mean estimation, we show that most elements in $\{\mathbf{Q}_1, \ldots, \mathbf{Q}_k\}$ are not far away from the truncated mean $\mu_Y$. 
\begin{lem}\label{lem:group}
	There exists a constant $C_0$. For any $\beta>0$, let $k=800\ln(1/\beta)$. Define
	\begin{eqnarray}
		r_0=C_0\left(\sqrt{\frac{d+\ln(1/\beta)}{n}}+\frac{R}{n}\sqrt{\frac{k}{\rho}} \sqrt{d+\ln(1/\beta)}\right).
		\label{eq:r0}
	\end{eqnarray}
	Then with probability at least $1-\beta$, 
	\begin{eqnarray}
		\underset{\mathbf{u}:\norm{\mathbf{u}}=1}{\sup} \sum_{j=1}^k \mathbf{1}\left(\langle \mathbf{u}, \mathbf{Q}_j-\mu_Y\rangle > r_0\right) \leq \frac{1}{10}k.
		\label{eq:concentrate}
	\end{eqnarray}
\end{lem}

\cite{lugosi2019sub} shows that for the non-private case, for arbitrary unit vector $\mathbf{u}$, most of elements in $\{\mathbf{Q}_1,\ldots, \mathbf{Q}_k\}$ satisfy $\langle \mathbf{u}, \mathbf{Q}_j-\mu\rangle\lesssim \sqrt{(d+\ln(1/\beta))/n}$, which matches the first term in \eqref{eq:concentrate}. Compared with \cite{lugosi2019sub}, we extend the analysis to the case with the clipping operation and random noise. 

\emph{2) Distance and gradient estimation (step \ref{step:est}).} We introduce the following optimization problem:
\begin{eqnarray}
	\begin{aligned}
		& \text{max}
		& & s \\
		& \text{subject to}
		& & b_j\langle \mathbf{Q}_j-\mathbf{c}, \mathbf{u}\rangle \geq b_j s, j=1,\ldots, k\\
		&& & \sum_{j=1}^k b_j \geq 0.9k\\
		&& & b_j\in \{0,1\}, j=1,\ldots, k,\\
		&&& \norm{\mathbf{u}} = 1.
	\end{aligned}
	\label{eq:opt}
\end{eqnarray}
Now we explain \eqref{eq:opt}. This optimization problem attempts to find maximum $s$, such that there exists a unit vector $\mathbf{u}$, at least $90\%$ of the projection of $\{\mathbf{Q}_1, \ldots, \mathbf{Q}_k \}$ on $\mathbf{u}$ are at least $s$ far away from the current iterate $\mathbf{c}$. Denote the function of estimation as $d, \mathbf{g}=\Est(\mathbf{Q}_1,\ldots, \mathbf{Q}_k, \mathbf{c})$. The program $\Est$ solves the optimization problem \eqref{eq:opt}, and returns $d$ and $\mathbf{g}$, in which $d$ is the optimal value of $s$, which estimates the distance $\norm{\mathbf{c}-\mu_Y}$, and $\mathbf{g}$ is the corresponding value of $\mathbf{u}$, which estimates the direction from $\mathbf{c}$ to $\mu_Y$, i.e. $(\mathbf{c}-\mu_Y)/\norm{\mathbf{c}-\mu_Y}$. Such estimation is analyzed in the following lemma, in which we follow the analysis in \cite{cherapanamjeri2019fast}.
\begin{lem}\label{lem:distest}
	Let $d, \mathbf{g}=\Est (\mathbf{Q}_1,\ldots, \mathbf{Q}_k, \mathbf{c})$, which is calculated by solving the optimization problem \eqref{eq:opt}. If \eqref{eq:concentrate} is satisfied, then 
	\begin{eqnarray}
		|d-\norm{\mathbf{c}-\mu_Y}|\leq r_0.
		\label{eq:derr}
	\end{eqnarray}
	Moreover, if $\norm{\mu_Y-\mathbf{c}}\geq 4r_0$, then
	\begin{eqnarray}
		\left\langle \mathbf{g}, \frac{\mu_Y-\mathbf{c}}{\norm{\mu_Y-\mathbf{c}}}\right\rangle \geq \frac{1}{2}.
		\label{eq:gerr}
	\end{eqnarray}
\end{lem}

\eqref{eq:derr} shows that under \eqref{eq:concentrate}, which holds with probability at least $1-\beta$, the error of distance estimate is at least $r_0$. Moreover, \eqref{eq:gerr} shows that if the current iterate $\mathbf{c}$ is sufficiently far away from the truncated mean $\mu_Y$, then the angle between gradient estimate and the ideal update direction along $\mu_Y-\mathbf{c}$ is no more than $\pi/3$. These analysis validates that the update rule $\mathbf{c}_{l+1}=\mathbf{c}_l+\eta d_lg_l$ makes the iterate point $\mathbf{c}_{l}$ close to $\mu_Y$ with large $l$. To be more precise, we show the following lemma:

\begin{lem}\label{lem:iterate}
	Let $\eta=1/4$. If $\norm{\mathbf{c}_1-\mu_Y}\leq 4r_0$, or $t_c\geq 2\ln \frac{\norm{\mathbf{c}_1-\mu_Y}}{4r_0}/\ln \frac{256}{233}$, then the estimate $\hat{\mu}$ with Algorithm \ref{alg:est_improve} satisfies 
	\begin{eqnarray}
		\norm{\hat{\mu}-\mu_Y}\leq 6r_0,
	\end{eqnarray}
	in which $\mu_Y=\mathbb{E}[\mathbf{Y}_i]$.
\end{lem}
Lemma \ref{lem:iterate} bounds the estimation error with respect to the truncated mean $\mu_Y$. Recall the definition of $r_0$ in \eqref{eq:r0}. Considering the clipping bias, we have
\begin{eqnarray}
	\norm{\hat{\mu}-\mu}\lesssim \sqrt{\frac{d+\ln(1/\beta)}{n}}+\frac{R}{n}\sqrt{\frac{k}{\rho}} \sqrt{d+\ln(1/\beta)}+ d^\frac{p}{2}R^{1-p}.
	\label{eq:meanerr}
\end{eqnarray}

\eqref{eq:mean} can then be obtained using \eqref{eq:meanerr} and \eqref{eq:rho}. The construction of the mean estimator and the proof of Theorem \ref{thm:mean} is complete.

\subsection{Application in DP Optimization}\label{sec:toopt}

Now we have constructed a mean estimator under $(\epsilon, \delta)$-DP, whose estimation error is bounded with Theorem \ref{thm:mean}. For the stochastic optimization problem, we need to estimate the gradient for $T$ steps, and the $(\epsilon, \delta)$-DP requirement is imposed on the whole process. Therefore, each step needs to satisfy stronger privacy requirements. According to advanced composition theorem (Lemma \ref{lem:basics}(1)), here we ensure that each step satisfies $(\epsilon_0, \delta_0)$-DP, with
\begin{eqnarray}
	\epsilon_0=\frac{\epsilon}{2\sqrt{2T\ln \frac{2}{\delta}}}, \delta_0=\frac{\delta}{2T}.
	\label{eq:dpstep}
\end{eqnarray}
Then the optimization process with $T$ steps is $(\epsilon, \delta)$-DP. For any fixed $\mathbf{w}$, let $g(\mathbf{w})$ be the mean estimate using $\nabla l(\mathbf{w}, \mathbf{Z}_1), \ldots, \nabla l(\mathbf{w}, \mathbf{Z}_n)$ under $(\epsilon_0, \delta_0)$-DP. According to Theorem \ref{thm:mean}, for any fixed $\mathbf{w}$, the gradient estimate at each step satisfies
\begin{eqnarray}
	\norm{g(\mathbf{w})-\nabla F(\mathbf{w})}\lesssim \sqrt{\frac{d+\ln \frac{1}{\beta}}{n}}  +d^\frac{p}{2} R^{1-p}+ \frac{R}{n\epsilon_0}\sqrt{d+\ln \frac{1}{\beta}}\left(\ln \frac{1}{\delta_0}+\sqrt{\ln \frac{1}{\delta_0}\ln\ln \frac{1}{\beta}}\right).
\end{eqnarray}

As discussed earlier, since $\mathbf{w}_t$ depends on the data, the bias and variance of gradient estimation at time $t$, i.e. $B=\max_t\norm{\mathbb{E}[g(\mathbf{w}_t)]-\nabla F(\mathbf{w}_t)}$ and $G^2=\max_t \mathbb{E}[\norm{g(\mathbf{w}_t)-\nabla F(\mathbf{w}_t)}^2]$ can not be bounded simply using the bias and variance with fixed $\mathbf{w}$. Instead, similar to Lemma \ref{lem:bg}, we need to derive a union bound of all $\mathbf{w}\in \mathcal{W}$. The results are shown in Lemma \ref{lem:bgnew}.
\begin{lem}\label{lem:bgnew}
	For the mean estimation algorithm, $B$ and $G^2$ are bounded by
	\begin{eqnarray}
		B\lesssim \sqrt{\frac{d}{n}}+d^\frac{p}{2} R^{1-p},
	\end{eqnarray}
	and
	\begin{eqnarray}
		G^2\lesssim \frac{d}{n}+\frac{R^2 d}{n^2\epsilon_0^2}\left(\ln \frac{1}{\delta_0}+\sqrt{\ln \frac{1}{\delta_0}\ln d}\right)+d^p R^{2(1-p)}.
		\label{eq:g2new}
	\end{eqnarray}
\end{lem}

Based on Lemma \ref{lem:opt} and \ref{lem:bgnew}, we can then derive the following bound on the excess risk.
\begin{thm}\label{thm:risknew}
	Let $T= n^2\epsilon^2/(dR^2)$, $\eta=1/\sqrt{2T\lambda^2}$, and
	$R=\sqrt{d}\left(n\epsilon/\sqrt{d}\right)^\frac{1}{p}$.
	If $\epsilon\leq 1$, then
	\begin{eqnarray}
		\mathbb{E}[F(\hat{\mathbf{w}})]-F(\mathbf{w}^*)\lesssim \sqrt{\frac{d}{n}}+\sqrt{d}\left(\frac{\sqrt{d}}{n\epsilon}\right)^{1-\frac{1}{p}}\ln \frac{1}{\delta}\ln(nd).
	\end{eqnarray}
\end{thm}

The proof of Theorem \ref{thm:risknew} is shown in Appendix H. The bound shown in Theorem \ref{thm:risknew} matches the minimax lower bound in Theorem \ref{thm:lower}, indicating that the new method is minimax rate optimal.  

\section{Conclusion}

In this paper, we have improved the convergence of population risk of stochastic optimization under DP. We have proposed two methods. The simple clipping method is relatively convenient to implement. It achieves $\tilde{O}\left(\sqrt{\frac{d}{n}}+\sqrt{d}\left(\frac{\sqrt{d}}{\epsilon n}\right)^{1-\frac{1}{p}}+\frac{d^{\frac{3}{2}-\frac{1}{p}}}{n^{1-\frac{1}{p}}} \right)$ risk bound. The iterative updating method further improves the risk bound to $\tilde{O}\left(\sqrt{\frac{d}{n}}+\sqrt{d}\left(\frac{\sqrt{d}}{\epsilon n}\right)^{1-\frac{1}{p}} \right)$, which matches the minimax lower bound, indicating that this method is optimal. 

\bibliographystyle{ieeetr}
\bibliography{DPoptimization}

\newpage

\appendices

\section{Proof of Lemma \ref{lem:opt}}\label{sec:opt}
\begin{eqnarray}
	\norm{\mathbf{w}_{t+1}-\mathbf{w}^*}^2 &=& \norm{\mathbf{w}_t-\eta g(\mathbf{w}_t)-\mathbf{w}^*}^2\nonumber\\
	&=& \norm{\mathbf{w}_t-\mathbf{w}^*}^2-2\eta \langle g(\mathbf{w}_t), \mathbf{w}_t-\mathbf{w}^*\rangle+\eta^2 \norm{g(\mathbf{w}_t)}^2.
\end{eqnarray}
Note that
\begin{eqnarray}
	\mathbb{E}[\norm{g(\mathbf{w}_t)}^2]&\leq& 2\mathbb{E}[\norm{\nabla F(\mathbf{w}_t)}^2]+2\mathbb{E}[\norm{g(\mathbf{w}_t)-\nabla F(\mathbf{w}_t)}^2]\nonumber\\
	&\leq & 2\lambda^2 L^2+2G^2.
	\label{eq:g2}
\end{eqnarray}
The last step holds because from Assumption \ref{ass}, $F$ is $\lambda$-smooth, and the diameter of $\mathcal{W}$ is bounded by $L$, thus for any $\mathbf{w}\in \mathcal{W}$,
\begin{eqnarray}
	\norm{\nabla F(\mathbf{w})}\leq \lambda \norm{\mathbf{w}-\mathbf{w}^*}\leq \lambda L.
\end{eqnarray} 
Moreover, from \eqref{eq:gdef}, $\mathbb{E}[\norm{g(\mathbf{w}_t)-\nabla F(\mathbf{w}_t)}^2]= G^2$. Therefore \eqref{eq:g2} holds.

Hence
\begin{eqnarray}
	\mathbb{E}[\norm{\mathbf{w}_{t+1}-\mathbf{w}^*}^2]&=&\mathbb{E}[\norm{\mathbf{w}_t-\mathbf{w}^*}^2]-2\eta \mathbb{E}[\langle F(\mathbf{w}_t), \mathbf{w}_t-\mathbf{w}^*\rangle]\nonumber\\
	&&+2\eta B \mathbb{E}[\norm{\mathbf{w}_t-\mathbf{w}^*}]+2\eta^2 (\lambda^2L^2+G^2).
	\label{eq:iter}
\end{eqnarray}
Now we reorganize \eqref{eq:iter}. Moreover, note that since $F$ is convex, $\langle \nabla F(\mathbf{w}), \mathbf{w}-\mathbf{w}^*\rangle \geq F(\mathbf{w})-F(\mathbf{w}^*)$. Therefore
\begin{eqnarray}
	\mathbb{E}[F(\mathbf{w}_t)]-F(\mathbf{w}^*)\leq \frac{1}{2\eta}\mathbb{E}[\norm{\mathbf{w}_t-\mathbf{w}^*}^2]-\frac{1}{2\eta} \mathbb{E}[\norm{\mathbf{w}_{t+1}-\mathbf{w}^*}^2]+LB+\eta(\lambda^2 L^2 + G^2).
\end{eqnarray}
Take average over $t=1,\ldots, T$,
\begin{eqnarray}
	\frac{1}{T}\sum_{t=1}^T \mathbb{E}[F(\mathbf{w}_t)]-F(\mathbf{w}^*)\leq \frac{L^2}{2\eta T}+LB+\eta(\lambda^2L^2+G^2).
\end{eqnarray}
Recall that $\hat{\mathbf{w}}=(1/T)\sum_{t=1}^T \mathbf{w}_t$. By Jensen's inequality, $F(\hat{\mathbf{w}})\leq (1/T)\sum_{t=1}^T F(\mathbf{w}_t)$. The proof is complete.

\section{Proof of Theorem \ref{thm:meanest}}\label{sec:meanest}

Denote $\mu_Y:=\mathbb{E}[\mathbf{Y}]$. The estimation error can be decomposed as follows:
\begin{eqnarray}
	\norm{\hat{\mu}-\mu}\leq \norm{\bar{\mathbf{Y}}-\mu_Y}+\norm{\mu_Y-\mu}+\norm{\mathbf{W}},
\end{eqnarray}
in which $\bar{\mathbf{Y}}=(1/n)\sum_{i=1}^n \mathbf{Y}_i$ is the average of $\mathbf{Y}_i$. The first term is caused by the randomness of samples. The second term is the clipping bias. The third term is the strength of additional noise for privacy protection. We provide a high probability bound of these three terms separately.

\textbf{Bound of $\norm{\bar{\mathbf{Y}}-\mu_Y}$.} The high probability bound of mean of random variables under bounded moment assumption has been analyzed in existing works, see for example, \cite{michel1976nonuniform}, Lemma 13 in \cite{zhao2024huber} and Corollary 3.5 in \cite{agarwal2024private}. For completeness and consistency of notations, we show the high probability bounds below.

Let $\mathbf{Y}$ be a random variable that is i.i.d with $\mathbf{Y}_1,\ldots, \mathbf{Y}_n$. Then
\begin{eqnarray}
	\mathbb{E}[\langle \mathbf{u}, \mathbf{Y}\rangle] = \langle \mathbf{u}, \mu_Y\rangle.
\end{eqnarray}
Define
\begin{eqnarray}
	V:=\langle \mathbf{u}, \mathbf{Y}-\mu_Y\rangle.
\end{eqnarray}
From \eqref{eq:yi}, 
\begin{eqnarray}
	|V|\leq |\langle\mathbf{u},\mathbf{Y}\rangle|+|\langle \mathbf{u}, \mu_Y\rangle|\leq 2R.
\end{eqnarray}
Moreover, $\mathbb{E}[V]=0$, and
\begin{eqnarray}
	\mathbb{E}[V^2]=\mathbb{E}[\langle \mathbf{u}, \mathbf{Y}-\mu_Y\rangle^2]\leq \mathbb{E}[\langle \mathbf{u}, \mathbf{Y}-\mu\rangle^2]\leq \mathbb{E}[\langle \mathbf{u}, \mathbf{X}-\mu\rangle^2]\leq M^2,
\end{eqnarray}
in which the last step comes from the condition $\mathbb{E}[|\langle \mathbf{u}, \mathbf{X}\rangle|^p]\leq M^p$.

By Lemma \ref{lem:mgf}, 
\begin{eqnarray}
	\mathbb{E}[e^{\lambda V}]\leq \exp\left(\frac{3}{4}\lambda^2 M^2\right), \forall 0\leq \lambda\leq \frac{1}{2R}.
\end{eqnarray}
Define $V_i=\langle \mathbf{u}, \mathbf{Y}_i-\mu_Y\rangle$. Then $V$ is i.i.d with $V_i$ for all $i=1,\ldots, n$. 
Correspondingly, let $\bar{V}=(1/n)\sum_{i=1}^n V_i$ be the average, then 
\begin{eqnarray}
	\mathbb{E}[e^{\lambda \bar{V}}]\leq \exp\left(\frac{3\lambda^2}{4n}M^2\right), \forall 0\leq \lambda\leq \frac{n}{2R}.
\end{eqnarray}
Therefore
\begin{eqnarray}
	\text{P}(\bar{V}>t)&\leq& \underset{0\leq \lambda\leq 1/(2R)}{\inf} e^{-\lambda t}\exp\left(\frac{3\lambda^2}{4n}M^2\right) \nonumber\\
	&\leq & \left\{
	\begin{array}{ccc}
		\exp\left(-\frac{nt^2}{3M^2}\right) & \text{if} & t\leq \frac{3M^2}{4R}\\
		\exp\left(-\frac{nt}{4R}\right) &\text{if} & t>\frac{3M^2}{4R}. 
	\end{array}
	\right.
	\label{eq:vtail}
\end{eqnarray}
\eqref{eq:vtail} indicates that $\bar{V}$ has subgaussian decaying probability at small $t$, and subexponential tail at large $t$. 

Based on the bound of $V_i$, we then derive a high probability bound of $\norm{\bar{\mathbf{Y}}-\mu_Y}$, in which $\bar{\mathbf{Y}}=(1/n)\sum_{i=1}^n \mathbf{Y}_i$. Here we use some arguments from eq.(104) to eq.(106) in Appendix B.1 in \cite{zhao:aaai:24}. Let $\mathbf{u}_j$, $j=1,\ldots, N_0$ be a $1/2$-covering of unit sphere. From Lemma 5.2 and Lemma 5.3 in \cite{vershynin2012introduction}, $N_0\leq 6^d$, and $\norm{\mathbf{x}}\leq 2\max_j\langle \mathbf{u}_j, \mathbf{x}\rangle$ for all vector $\mathbf{x}$. Therefore, from \eqref{eq:vtail},
\begin{eqnarray}
	\text{P}(\norm{\bar{\mathbf{Y}}-\mu_Y}>t)&\leq & \text{P}\left(2\max_j \langle \mathbf{u}_j, \bar{\mathbf{Y}}-\mu_Y\rangle > t\right) \nonumber\\
	&\leq & N_0\text{P}\left(\langle \mathbf{u}_j, \bar{\mathbf{Y}}-\mu_Y\rangle > \frac{t}{2}\right)\nonumber\\
	&\leq & \left\{
	\begin{array}{ccc}
		6^d \exp\left(-\frac{nt^2}{12M^2}\right) &\text{if} & t\leq \frac{3M^2}{2R}\\
		6^d \exp\left(-\frac{nt}{8R}\right) &\text{if} & t>\frac{3M^2}{2R}.
	\end{array}
	\right.
\end{eqnarray}
Therefore, with probability $1-\beta/2$,
\begin{eqnarray}
	\norm{\bar{\mathbf{Y}}-\mu_Y}\leq \max\left\{\sqrt{\frac{12M^2}{n}\ln \frac{2\times 6^d}{\beta}}, \frac{8R}{n}\ln\frac{2\times 6^d}{\beta}\right\}.
	\label{eq:b1}
\end{eqnarray}

\textbf{Bound of $\norm{\mu_Y-\mu}$.} This term is the clipping bias. We omite $i$ in the following steps. From \eqref{eq:yi},
\begin{eqnarray}
	\mu_Y=\mathbb{E}[\mathbf{Y}]=\mathbb{E}[\mathbf{X}\mathbf{1}(\norm{\mathbf{X}}\leq R)]+\mathbb{E}\left[\mathbf{X}\frac{R}{\norm{\mathbf{X}}}\mathbf{1}(\norm{\mathbf{X}}>R)\right].
\end{eqnarray} 
Moreover,
\begin{eqnarray}
	\mu=\mathbb{E}[\mathbf{X}] = \mathbb{E}[\mathbf{X}\mathbf{1}(\norm{\mathbf{X}}\leq R)]+\mathbb{E}[\mathbf{X}\mathbf{1}(\norm{\mathbf{X}}>R)].
\end{eqnarray}
Thus
\begin{eqnarray}
	\norm{\mu_Y-\mu} &\leq & \mathbb{E}\left[\norm{\mathbf{X}-\mathbf{X}\frac{R}{\norm{\mathbf{X}}}}\mathbf{1}(\norm{\mathbf{X}}>R)\right]\nonumber\\
	&=&\mathbb{E}[(\norm{\mathbf{X}}-R)\mathbf{1}(\norm{\mathbf{X}}>R)]\nonumber\\
	&=&\int_0^\infty \text{P}(\norm{\mathbf{X}}>R+t)dt\nonumber\\
	&\overset{(a)}{\leq}&\int_0^\infty \frac{d^\frac{p}{2}M^p}{(R+t)^p} dt\nonumber\\
	&=&\frac{d^\frac{p}{2}M^p}{p-1}R^{1-p},
	\label{eq:b2}
\end{eqnarray}
in which (a) comes from Lemma \ref{lem:nondir}.

\textbf{Bound of $\norm{\mathbf{W}}$.} For any vector $\mathbf{u}$, $\langle \mathbf{u}, \mathbf{W}\rangle$ is subgaussian with parameter $\sigma^2$. From the property of subgaussian distribution,
\begin{eqnarray}
	\text{P}(\langle \mathbf{u}, \mathbf{W}\rangle >t)\leq e^{-\frac{t^2}{2\sigma^2}}.
\end{eqnarray}
Let $\mathbf{u}_j$, $j=1,\ldots, N_0$ be a $1/2$-covering. Then
\begin{eqnarray}
	\text{P}(\norm{\mathbf{W}}>t)\leq \text{P}\left(2\max_j\langle \mathbf{u}_j, \mathbf{W}\rangle >t\right)\leq N_0 \text{P}\left(\langle \mathbf{u}_j, \mathbf{W}\rangle > \frac{t}{2}\right) \leq 6^d e^{-\frac{t^2}{8\sigma^2}}.
\end{eqnarray}
Therefore, with probability $1-\beta/2$,
\begin{eqnarray}
	\norm{\mathbf{W}}\leq \sqrt{8\sigma^2 \ln \frac{2\times 6^d}{\beta}}=\sqrt{\frac{16R^2}{\rho n^2}\ln \frac{2\times 6^d}{\beta}}.
	\label{eq:b3}
\end{eqnarray}
Finally, we combine \eqref{eq:b1}, \eqref{eq:b2} and \eqref{eq:b3}. Note that \eqref{eq:b1} and \eqref{eq:b3} hold with probability $1-\beta/2$, and \eqref{eq:b2} holds for sure. Therefore, with probability $1-\beta$,

\begin{eqnarray}
	\norm{\hat{\mu}-\mu}\leq \max\left\{\sqrt{\frac{12M^2}{n}\ln \frac{2\times 6^d}{\beta}}, \frac{8R}{n}\ln \frac{2\times 6^d}{\beta} \right\}+\frac{d^\frac{p}{2}M^p}{p-1}R^{1-p}+\frac{4R}{n\sqrt{\rho}}\sqrt{\ln \frac{2\times 6^d}{\beta}}.
\end{eqnarray}

\section{Proof of Lemma \ref{lem:bg}}\label{sec:bg}
Lemma \ref{lem:bg} provides bounds of both $B$ and $G^2$. We bound them separately.

\textbf{Bound of $B$.} For any fixed $\mathbf{w}$, $\norm{\mathbb{E}[g(\mathbf{w})]-\nabla F(\mathbf{w})}$ can be easily bounded using Theorem \ref{thm:meanest}. However, $\mathbf{w}_t$ depends on the data, therefore $\norm{\mathbb{E}[g(\mathbf{w}_t)]-\nabla F(\mathbf{w}_t)}$ can not be directly derived using the bound of $\norm{\mathbb{E}[g(\mathbf{w})]-\nabla F(\mathbf{w})}$. Define
\begin{eqnarray}
	g_0(\mathbf{w})=\frac{1}{n}\sum_{i=1}^n \Clip(\nabla l(\mathbf{w}, \mathbf{Z}_i), R).
	\label{eq:g0}
\end{eqnarray}
Then from Algorithm \ref{alg:est} and \ref{alg:so}, $g(\mathbf{w})=g_0(\mathbf{w}) + \mathbf{W}$, in which $\mathbf{W}\sim \mathcal{N}(0, \sigma^2)$, with $\sigma^2=2R^2 T/(\rho n^2)$. Then
\begin{eqnarray}
	\norm{\mathbb{E}[g(\mathbf{w}_t)]-\nabla F(\mathbf{w}_t)} &\overset{(a)}{=} & \norm{\mathbb{E}[g_0(\mathbf{w}_t)]-\nabla F(\mathbf{w}_t)}\nonumber\\
	&\overset{(b)}{\leq} & \mathbb{E}\left[\norm{g_0(\mathbf{w}_t)-\nabla F(\mathbf{w}_t)}\right]\nonumber\\
	&\leq & \mathbb{E}\left[\sup_\mathbf{w} \norm{g_0(\mathbf{w})-\nabla F(\mathbf{w})}\right].
	\label{eq:bconv}
\end{eqnarray}
(a) holds because the noise $\mathbf{W}$ does not depend on previous steps. (b) uses Jensen's inequality.

It remains to bound the right hand side of \eqref{eq:bconv}. From Theorem \ref{thm:meanest}, with probability $1-\beta$,
\begin{eqnarray}
	\norm{g_0(\mathbf{w})-\nabla F(\mathbf{w})}\leq \max\left\{\sqrt{\frac{12M^2}{n}\ln \frac{2\times 6^d}{\beta}},\frac{8R}{n}\ln \frac{2\times 6^d}{\beta} \right\}+\frac{d^\frac{p}{2}M^p}{p-1}R^{1-p}.
\end{eqnarray}
Let $\mathbf{c}_j$, $j=1,\ldots, N_c$ be a $a$-covering of $\mathcal{W}$. Then there exists a constant $C_0$, such that
\begin{eqnarray}
	N_c\leq C_0\left(\frac{M}{a}\right)^d.
\end{eqnarray}
Therefore, with probability $1-\beta$,
\begin{eqnarray}
	\max_j\norm{g_0(\mathbf{c}_j)-\nabla F(\mathbf{c}_j)}\leq \max\left\{\sqrt{\frac{12M^2}{n}\ln \frac{2\times 6^d N_c}{\beta}}, \frac{8R}{n}\ln \frac{2\times 6^d N_c}{\beta} \right\}+\frac{d^\frac{p}{2}M^p}{p-1}R^{1-p}.
	\label{eq:supj}
\end{eqnarray}

To bound the right hand side of \eqref{eq:bconv}, we show that $g_0$ is $\lambda$-Lipschitz. From Assumption \ref{ass}, since $l$ is $\lambda$-smooth with respect to $\mathcal{W}$, 
\begin{eqnarray}
	|g_0(\mathbf{w})-g_0(\mathbf{w}')| &=& \left|\frac{1}{n}\sum_{i=1}^n \left[\Clip(\nabla l(\mathbf{w}, \mathbf{Z}_i), R) - \Clip(\nabla l(\mathbf{w}', \mathbf{Z}_i), R)\right]\right|\nonumber\\
	&\leq &  \frac{1}{n}\sum_{i=1}^n \left|\Clip(\nabla l(\mathbf{w}, \mathbf{Z}_i), R) - \Clip(\nabla l(\mathbf{w}', \mathbf{Z}_i), R)\right|\nonumber\\
	&\leq & \lambda \norm{\mathbf{w}-\mathbf{w}'},
\end{eqnarray}
therefore $g_0$ is $\lambda$-Lipschitz. Since $F$ is $\lambda$-smooth, $\nabla F$ is also $\lambda$-Lipschitz. Therefore, with probability $1-\beta$,
\begin{eqnarray}
	\sup_{\mathbf{w}\in \mathcal{W}} \norm{g_0(\mathbf{w})-\nabla F(\mathbf{w})}\leq \max\left\{\sqrt{\frac{12M^2}{n}\ln \frac{2\times 6^d N_c}{\beta}}, \frac{8R}{n}\ln \frac{2\times 6^d N_c}{\beta} \right\}+\frac{d^\frac{p}{2}M^p}{p-1}R^{1-p}+2La.
	\label{eq:sup}
\end{eqnarray}
Let $\beta=1/(n^2)$, $a=1/n$, then $\ln(2\times 6^d N_c/\beta)\lesssim d\ln n$. From \eqref{eq:g0}, $\norm{g_0(\mathbf{w})}\leq R$, thus if \eqref{eq:sup} is violated (which happens with probability no more than $\beta$), then $\norm{g_0(\mathbf{w}) - \nabla F(\mathbf{w})}$ can still be bounded by $R$. Therefore
\begin{eqnarray}
\mathbb{E}\left[\sup_\mathbf{w} \norm{g_0(\mathbf{w})-\nabla F(\mathbf{w})}\right]\lesssim \sqrt{\frac{d\ln n}{n}}\vee \frac{Rd}{n}\ln n+d^\frac{p}{2}R^{1-p}.	
\end{eqnarray}
From \eqref{eq:bconv}, the proof of the bound of $B$ in \eqref{eq:bbound} is complete.

\textbf{Bound of $G^2$.}
\begin{eqnarray}
	\mathbb{E}\left[\norm{g(\mathbf{w}_t)-\nabla F(\mathbf{w}_t)}^2\right] &\leq & \mathbb{E}\left[\norm{g_0(\mathbf{w}_t)-\nabla F(\mathbf{w}_t)}^2\right] + \sigma^2 d\nonumber\\
	&\leq & \mathbb{E}\left[\left(\sup_\mathbf{w} \norm{g_0(\mathbf{w}) - \nabla F(\mathbf{w})}^2\right)^2\right]+\sigma^2 d\nonumber\\
	&\lesssim & \frac{d\ln n}{n}+\frac{R^2d^2}{n^2}\ln^2 n + d^p R^{2(1-p)}+\frac{R^2Td}{\rho n^2}.
\end{eqnarray}

The proof of the bound of $G^2$ in \eqref{eq:gbound} is complete.

\section{Proof of Theorem \ref{thm:risk}}\label{sec:risk}
From Lemma \ref{lem:opt} and Lemma \ref{lem:bg},
\begin{eqnarray}
	\mathbb{E}[F(\hat{\mathbf{w}})] - F(\mathbf{w}^*)&\leq & \frac{L^2}{2\eta T}+LB+\eta(\lambda^2L^2+G^2)\nonumber\\
	&=&\frac{L^2}{2T}\sqrt{2T\lambda^2}+LB+\sqrt{\frac{1}{2T\lambda^2}}(L^2\lambda^2+G^2)\nonumber\\
	&=&L^2\lambda\sqrt{\frac{2}{T}}+LB+\frac{G^2}{\sqrt{2T\lambda^2}}\nonumber\\
	&\lesssim & \frac{1}{\sqrt{T}}+\sqrt{\frac{d\ln n}{n}}+\frac{Rd}{n}\ln n+d^\frac{p}{2}R^{1-p}\nonumber\\
	&&+\frac{1}{\sqrt{T}}\left(\frac{d\ln n}{n}+\frac{R^2d^2}{n^2}\ln^2 n + d^p R^{2(1-p)}+\frac{R^2 Td}{\rho n^2}\right)\nonumber\\
	&\lesssim & \frac{R\sqrt{d}}{n\sqrt{\rho}}+\sqrt{\frac{d\ln n}{n}}+\frac{Rd}{n}\ln n+d^\frac{p}{2}R^{1-p}.
	\label{eq:risktmp}
\end{eqnarray}
To minimize \eqref{eq:risktmp}, let
\begin{eqnarray}
	R=\sqrt{d}\left(\frac{n\sqrt{\rho}}{\sqrt{d}}\right)^\frac{1}{p}\wedge \sqrt{d}\left(\frac{n}{d}\right)^\frac{1}{p}.
\end{eqnarray}
If $\rho>1/d$, then
\begin{eqnarray}
	\mathbb{E}[F(\hat{\mathbf{w}})]-F(\mathbf{w}^*)\lesssim \sqrt{\frac{d\ln n}{n}}+\frac{d^{\frac{3}{2}-\frac{1}{p}}}{n^{1-\frac{1}{p}}}\ln n.
\end{eqnarray}
If $\rho\leq 1/d$, then
\begin{eqnarray}
	\mathbb{E}[F(\hat{\mathbf{w}})]-F(\mathbf{w}^*) \lesssim \sqrt{\frac{d\ln n}{n}}+\sqrt{d}\left(\frac{\sqrt{d}}{n\sqrt{\rho}}\right)^{1-\frac{1}{p}}.
\end{eqnarray}
Combining these two cases, the overall excess risk is bounded by
\begin{eqnarray}
	\mathbb{E}[F(\hat{\mathbf{w}})]-F(\mathbf{w}^*)\lesssim \sqrt{\frac{d\ln n}{n}}+\sqrt{d}\left(\frac{\sqrt{d}}{n\sqrt{\rho}}\right)^{1-\frac{1}{p}}+\frac{d^{\frac{3}{2}-\frac{1}{p}}}{n^{1-\frac{1}{p}}}\ln n.
\end{eqnarray}

\section{Proof of Theorem \ref{thm:dp}}\label{sec:dp}
Recall that $Q_j=(1/m)\sum_{i\in B_j}\mathbf{Y}_i+\mathbf{W}_j$ with $\mathbf{W}_j\sim \mathcal{N}(0, \sigma^2)$, $\sigma^2=2R^2/(\rho m^2)$. By Lemma \ref{lem:var}, $\mathbf{Q}_j$ is $\rho$-CDP for $j=1,\ldots, k$. From Lemma \ref{lem:basics}(4), for all $\delta_0>0$, $\mathbf{Q}_j$ is $(\epsilon_0, \delta_0)$-DP, with
\begin{eqnarray}
	\epsilon_0=\rho + 2\sqrt{\rho\ln \frac{1}{\delta_0}}.
	\label{eq:eps0}
\end{eqnarray}
Here we let $\delta_0=\delta/(12k)$. 

Since $\hat{\mu}$ is permutation invariant with respect to $\mathbf{Q}_1,\ldots, \mathbf{Q}_k$, we can equivalently claim that $\hat{\mu}$ is based on random shuffling of $\mathbf{Q}_1,\ldots, \mathbf{Q}_k$. From Theorem III.8 in \cite{feldman2022hiding}, $\hat{\mu}$ is $(\epsilon_g, \delta_g)$-DP, with
\begin{eqnarray}
	\epsilon_g=\min\left\{\ln \left[1+\frac{e^{\epsilon_0}-1}{e^{\epsilon_0}+1}\left(\frac{8\sqrt{e^{\epsilon_0}\ln \frac{4}{\delta'}}}{\sqrt{k}}+\frac{8e^{\epsilon_0}}{\sqrt{k}}\right)\right], \epsilon_0 \right\},
	\label{eq:epsg}
\end{eqnarray}
and
\begin{eqnarray}
	\delta_g=\delta'+(e^{\epsilon_g}+1)\left(1+\frac{1}{2}e^{-\epsilon_0}\right) k\delta_0.
	\label{eq:deltag}
\end{eqnarray}
\eqref{eq:epsg} and \eqref{eq:deltag} hold for any $\delta'\in (0,1)$. Note that \eqref{eq:epsg} is slightly different from \cite{feldman2022hiding}. In particular, we upper bound $\epsilon_g$ by $\epsilon_0$, which does not appear in Theorem III.8 in \cite{feldman2022hiding}. It holds because the post-processing property \cite{dwork2014algorithmic}, the shuffling operation is at least not harmful to the privacy, if it does not contribute to privacy amplification.

It remains to show that with proper selection of $\delta'$, $\epsilon_g\leq \epsilon$ and $\delta_g\leq \delta$ hold. Recall the condition \eqref{eq:epsrange} in the statement of Theorem \ref{thm:dp}. By \eqref{eq:rho},
\begin{eqnarray}
	\rho\leq \frac{\ln \frac{4}{\delta}}{\ln \frac{8}{\delta}\left(1+2\sqrt{\ln \frac{12k}{\delta}}\right)^2} < {\ln \frac{8}{\delta}\left(1+2\sqrt{\ln \frac{12k}{\delta}}\right)^2} < 1.
\end{eqnarray} 
Hence, from \eqref{eq:eps0},
\begin{eqnarray}
	\epsilon_0&\leq & \rho + 2\sqrt{\rho\ln \frac{1}{\delta_0}}\nonumber\\
	&\leq & \sqrt{\rho} (1+2\sqrt{\ln \frac{1}{\delta_0}})\nonumber\\
	&=&\sqrt{\rho} \left(1+2\sqrt{\ln \frac{12k}{\delta}}\right)\nonumber\\
	&\leq & 1.
\end{eqnarray}
 Let $\delta'=\delta/2$. Then
\begin{eqnarray}
	\delta_g &\leq& \frac{1}{2} \delta+(e+1)(1+\frac{1}{2})k\delta_0\nonumber\\
	&=& \frac{1}{2}\delta+6k\delta_0\nonumber\\
	&\leq & \delta.
\end{eqnarray}

 Moreover, since $\epsilon_0\leq 1$, we have $e^{\epsilon_0}\leq e$. From \eqref{eq:epsg}, $\epsilon_g\leq \epsilon_0$, which yields $e^{\epsilon_g}\leq e$. Hence
\begin{eqnarray}
	\epsilon_g&\leq& \frac{e^{\epsilon_0}-1}{e^{\epsilon_0}+1}\left(\frac{8\sqrt{e^{\epsilon_0}\ln \frac{4}{\delta'}}}{\sqrt{k}}+\frac{8e^{\epsilon_0}}{\sqrt{k}}\right)\nonumber\\
	&\leq & \frac{e^{\epsilon_0} \epsilon_0}{e^{\epsilon_0}+1}\left(\frac{8\sqrt{e^{\epsilon_0}\ln \frac{4}{\delta'}}}{\sqrt{k}}+\frac{8e^{\epsilon_0}}{\sqrt{k}}\right)\nonumber\\
	&\leq & \frac{e}{2}\epsilon_0\left(\frac{8\sqrt{e\ln \frac{8}{\delta}}}{\sqrt{k}}+\frac{8e}{k}\right)\nonumber\\
	&\leq & 8e^2\epsilon_0\sqrt{\frac{1}{k}\ln \frac{8}{\delta}}\nonumber\\
	&\leq & 8e^2 (\rho + 2\sqrt{\rho \ln \frac{1}{\delta_0}})\sqrt{\frac{1}{k}\ln \frac{8}{\delta}}\nonumber\\
	&\leq & 8e^2 \sqrt{\rho} \left(1+2\sqrt{\ln \frac{1}{\delta_0}}\right) \sqrt{\frac{1}{k}\ln \frac{8}{\delta}}\nonumber\\
	&\leq & \frac{\epsilon\sqrt{k}}{\sqrt{\ln \frac{8}{\delta}}\left(1+2\sqrt{\ln \frac{12k}{\delta}}\right)} \left(1+2\sqrt{\ln \frac{12k}{\delta}}\right)\sqrt{\frac{1}{k}\ln \frac{8}{\delta}}\nonumber\\
	&= &\epsilon.
\end{eqnarray}
Now it is shown that $\epsilon_g\leq \epsilon$ and $\delta_g\leq \delta$. Therefore, any estimator $\hat{\mu}$ that is based on $\mathbf{Q}_1,\ldots, \mathbf{Q}_k$ and is permutation invariant with respect to $\mathbf{Q}_1, \ldots, \mathbf{Q}_k$ satisfies the privacy requirement.

\section{Proof of Lemma \ref{lem:group}}\label{sec:group}
Our proof is partially inspired by the proof of Lemma 1 in \cite{lugosi2019sub}. 
 
Let $\sigma_i$ be Rademacher random variable, i.e. $\text{P}(\sigma_i=1)=\text{P}(\sigma_i=-1)=1/2$ for all $i=1,\ldots, k$. For some $r>0$ and some unit vector $\mathbf{u}$ with $\norm{\mathbf{u}}=1$,
\begin{eqnarray}
	&&\mathbb{E}\left[\sup_{\mathbf{u}:\norm{\mathbf{u}}=1}\frac{1}{k}\sum_{j=1}^k \mathbf{1}\left(\langle \mathbf{u}, \mathbf{Q}_j-\mu_Y\rangle > r\right)\right] \nonumber\\
	&\leq & \frac{1}{r}\mathbb{E}\left[\sup_{\mathbf{u}:\norm{\mathbf{u}}=1} \frac{1}{k}\sum_{j=1}^k |\langle \mathbf{u}, \mathbf{Q}_j-\mu_Y\rangle|\right]\nonumber\\
	&\overset{(a)}{\leq} & \frac{1}{r} \mathbb{E}\left[\sup_{\mathbf{u}:\norm{\mathbf{u}}=1} \frac{1}{k}\sum_{j=1}^k \left(|\langle \mathbf{u}, \mathbf{Q}_j-\mu_Y\rangle|-\mathbb{E}[|\langle \mathbf{u}, \mathbf{Q}_j-\mu_Y\rangle|]\right) \right] +\frac{1}{r}\sqrt{\frac{2R^2}{\rho m^2}+\frac{M^2}{m}}\nonumber\\
	&\overset{(b)}{\leq} & \frac{2}{r} \mathbb{E}\left[\sup_{\mathbf{u}:\norm{\mathbf{u}}=1} \frac{1}{k}\sum_{j=1}^k \sigma_j|\langle \mathbf{u}, \mathbf{Q}_j-\mu_Y\rangle|\right] + \frac{1}{r}\sqrt{\frac{2R^2}{\rho m^2}+\frac{M^2}{m}}\nonumber\\
	&\overset{(c)}{\leq} & \frac{2}{r} \mathbb{E}\left[\sup_{\mathbf{u}:\norm{\mathbf{u}}=1}\left|\frac{1}{k}\sum_{j=1}^k \sigma_j \langle \mathbf{u}, \mathbf{Q}_j-\mu_Y\rangle\right|\right] +\frac{1}{r}\sqrt{\frac{2R^2}{\rho m^2}+\frac{M^2}{m}}\nonumber\\
	&\overset{(d)}{\leq} & \frac{4}{r} \mathbb{E}\left[\sup_{\mathbf{u}:\norm{\mathbf{u}}=1} \left|\frac{1}{k}\sum_{j=1}^k \langle \mathbf{u}, \mathbf{Q}_j-\mu_Y\rangle\right|\right] +\frac{1}{r}\sqrt{\frac{2R^2}{\rho m^2}+\frac{M^2}{m}}\nonumber\\
	&=& \frac{4}{r} \mathbb{E}\left[\sup_{\mathbf{u}:\norm{\mathbf{u}}=1} \left|\frac{1}{n}\sum_{i=1}^n \langle \mathbf{u}, \mathbf{Y}_i-\mu_Y\rangle+ \frac{1}{k}\sum_{j=1}^k \langle \mathbf{u}, \mathbf{W}_j\rangle\right|\right] + \frac{1}{r}\sqrt{\frac{2R^2}{\rho m^2}+\frac{M^2}{m}}.
	\label{eq:larger}
\end{eqnarray}

(a) holds because
\begin{eqnarray}
	\mathbb{E}[|\langle \mathbf{u}, \mathbf{Q}_j-\mu_Y\rangle|]\leq \sqrt{\mathbb{E}\left[\langle \mathbf{u}, \mathbf{Q}_j-\mu_Y\rangle^2\right]}=\sqrt{\sigma^2 + \frac{M^2}{m}}=\sqrt{\frac{2R^2 k}{\rho n^2} + \frac{M^2}{m}}.
\end{eqnarray}

(b) uses \eqref{eq:rad1} in Lemma \ref{lem:rad}. (c) comes from Lemma \ref{lem:contraction}. (d) uses \eqref{eq:rad2} in Lemma \ref{lem:rad}. Note that from Lemma \ref{lem:nondir}, $\mathbb{E}[\norm{\mathbf{X}-\mu}^p]\leq d^\frac{p}{2}M^p$, thus
\begin{eqnarray}
	\mathbb{E}\left[\sup_{\mathbf{u}:\norm{\mathbf{u}}=1} \left|\frac{1}{n}\sum_{i=1}^n \langle \mathbf{u}, \mathbf{Y}_i-\mu_Y\rangle\right|\right] &=& \mathbb{E}\left[\sup_{\mathbf{u}:\norm{\mathbf{u}}=1} |\langle \mathbf{u}, \bar{\mathbf{Y}}-\mu_Y\rangle|\right]\nonumber\\
	&=&\mathbb{E}\left[\norm{\bar{\mathbf{Y}}-\mu_Y}\right]\nonumber\\
	 &\leq& \sqrt{\mathbb{E}\left[\norm{\bar{\mathbf{Y}}-\mu_Y}^2\right]}\nonumber\\
	&\leq & \sqrt{\mathbb{E}[\norm{\bar{Y}-\mu}^2]}\nonumber\\
	&\leq & \sqrt{\mathbb{E}[\norm{\bar{\mathbf{X}}-\mu}^2]}\nonumber\\
	&\leq & \sqrt{\frac{dM^2}{n}}.
\end{eqnarray}
Moreover,
\begin{eqnarray}
	\mathbb{E}\left[\sup_{\mathbf{u}:\norm{\mathbf{u}}=1} \frac{1}{k}\sum_{j=1}^k \langle \mathbf{u}, \mathbf{W}_j\rangle\right] &=& \mathbb{E}\left[\norm{\frac{1}{k}\sum_{j=1}^k \mathbf{W}_j}\right]\nonumber\\
	&=&\frac{1}{\sqrt{k}} \mathbb{E}[\norm{\mathbf{W}_j}]\nonumber\\
	&=& \frac{\sigma\sqrt{d}}{\sqrt{k}}\nonumber\\
	&=& \frac{\sqrt{2} R}{m\sqrt{k}}\sqrt{\frac{d}{\rho}}.
\end{eqnarray}
Therefore
\begin{eqnarray}
	\eqref{eq:larger} \leq \frac{4M}{r}\sqrt{\frac{d}{n}}+ \frac{4\sqrt{2}R}{rm\sqrt{k}}\sqrt{\frac{d}{\rho}}+\frac{1}{r}\sqrt{\frac{2R^2}{\rho m^2}+\frac{M^2}{m}}.
	\label{eq:larger2}
\end{eqnarray}
Define a random variable
\begin{eqnarray}
	V=\sup_{\mathbf{u}:\norm{\mathbf{u}}=1}\frac{1}{k}\sum_{j=1}^k \mathbf{1}(\langle \mathbf{u}, \mathbf{Q}_j-\mu_Y\rangle>r).
\end{eqnarray}
Then from \eqref{eq:larger} and \eqref{eq:larger2},
\begin{eqnarray}
	\mathbb{E}[V]\leq \frac{4M}{r}\sqrt{\frac{d}{n}}+ \frac{4\sqrt{2}R}{rm\sqrt{k}}\sqrt{\frac{d}{\rho}}+\frac{1}{r}\sqrt{\frac{2R^2}{\rho m^2}+\frac{M^2}{m}}.
\end{eqnarray}
By Bounded difference inequality \cite{boucheron2013concentration},
\begin{eqnarray}
	\text{P}(V-\mathbb{E}[V])\leq 2e^{-2kt^2}.
\end{eqnarray}
Let $t=\sqrt{\ln(1/\beta)/(2k)}$, then $\text{P}(V-\mathbb{E}[V]>\sqrt{\ln(1/\beta)/(2k)})\leq \beta$. Therefore, with probability at least $1-\beta$,
\begin{eqnarray}
	V\leq \frac{4M}{r}\sqrt{\frac{d}{n}}+\frac{4\sqrt{2} R}{rm\sqrt{k}}\sqrt{\frac{d}{\rho}}+\frac{1}{r}\sqrt{\frac{2R^2}{\rho m^2}+\frac{M^2}{m}}+\sqrt{\frac{1}{2k}\ln \frac{1}{\beta}}.
	\label{eq:vbound}
\end{eqnarray}
From the statement of Lemma \ref{lem:group}, now it remains to select $r$ and $k$ to ensure that $V\leq 1/10$. We can simply let each term in the right hand side of \eqref{eq:vbound} to be less than $1/40$, i.e.
\begin{eqnarray}
	\frac{4M}{r}\sqrt{\frac{d}{n}}\leq \frac{1}{40}, \frac{4\sqrt{2} R}{rm\sqrt{k}}\sqrt{\frac{d}{\rho}}\leq \frac{1}{40},\frac{1}{r}\sqrt{\frac{2R^2}{\rho m^2}+\frac{M^2}{m}}\leq \frac{1}{40}, \sqrt{\frac{1}{2k}\ln \frac{1}{\beta}}\leq \frac{1}{40}.
	\label{eq:conds}
\end{eqnarray}
Let $k=800\ln (1/\beta)$, and
\begin{eqnarray}
	r_0=\max\left\{160M \sqrt{\frac{d}{n}}, 160\sqrt{2}\frac{R\sqrt{k}}{n}\sqrt{\frac{d}{\rho}}, 40\sqrt{\frac{2R^2k^2}{\rho n^2}+\frac{M^2k}{n}} \right\},
\end{eqnarray}
Note that $n=mk$. It can be seen that the conditions in \eqref{eq:conds} are satisfied with $r=r_0$. 

Finally, we summarize the above analysis. It is shown that with probability at least $1-\beta$, for any $\mathbf{u}$ with $\norm{\mathbf{u}}=1$,

\begin{eqnarray}
	\frac{1}{k}\sum_{j=1}^k \mathbf{1}(\langle \mathbf{u}, \mathbf{Q}_j-\mu_Y\rangle>r)\leq \frac{1}{10}.
\end{eqnarray}
The proof of Lemma \ref{lem:group} is complete.

\section{Proof of Lemma \ref{lem:distest} and Lemma \ref{lem:iterate}}\label{sec:distest}
We first prove Lemma \ref{lem:distest} under \eqref{eq:concentrate}.

\textbf{Proof of \eqref{eq:derr}.} For convenience, denote $d=\Distest(\mathbf{Q}_1,\ldots, \mathbf{Q}_k, \mathbf{c})$. Let
\begin{eqnarray}
	\mathbf{u}=\frac{\mu_Y-\mathbf{c}}{\norm{\mu_Y-\mathbf{c}}}.
\end{eqnarray}
By \eqref{eq:concentrate}, for at least $0.9k$ points, $\langle \mathbf{u}, \mathbf{Q}_j-\mu_Y\rangle\leq r_0$. Thus
\begin{eqnarray}
	\langle \mathbf{Q}_j-\mathbf{c},\mathbf{u}\rangle &=& \langle \mathbf{Q}_j-\mu_Y, \mathbf{u}\rangle + \langle \mu_Y-\mathbf{c}, \mathbf{u}\rangle\nonumber\\
	&\leq & r_0+\norm{\mu_Y-\mathbf{c}}
\end{eqnarray}
for at least $0.9k$ points among $\{1,\ldots, k\}$. From the optimization problem \eqref{eq:opt}, it can be shown that 
\begin{eqnarray}
	d\leq r_0+\norm{\mu_Y-\mathbf{c}}.
	\label{eq:dub}
\end{eqnarray}
Moreover, by replacing $\mathbf{u}$ with $-\mathbf{u}$ in \eqref{eq:concentrate}, it can be found that for at least $0.9k$ points, $\langle -\mathbf{u}, \mathbf{Q}_j-\mu_Y\rangle\leq r_0$. Therefore
\begin{eqnarray}
	\langle \mathbf{Q}_j-\mathbf{c}, \mathbf{u}\rangle\geq -r_0+\norm{\mu_Y-\mathbf{c}}
\end{eqnarray}
for at least $0.9k$ points among $\{1,\ldots, k\}$. Therefore, $d\geq -r_0+\norm{\mu_Y-\mathbf{c}}$. Combined with the upper bound \eqref{eq:dub}, the proof of \eqref{eq:derr} is complete.

\textbf{Proof of \eqref{eq:gerr}.} Recall that Lemma \ref{lem:distest} makes the assumption that $\norm{\mu_Y-\mathbf{c}}\geq 4r_0$. From 
\eqref{eq:derr},
\begin{eqnarray}
	d\geq \norm{\mu_Y-\mathbf{c}}-r_0\geq \frac{3}{4}\norm{\mu_Y-\mathbf{c}}.
\end{eqnarray}
Since $(d, \mathbf{g})$ is the solution to the optimization problem in \eqref{eq:opt}, for at least $0.9k$ points, $\langle \mathbf{Q}_j-\mathbf{c}, \mathbf{g}\rangle \geq d$. Moreover, from the condition \eqref{eq:concentrate}, $\langle \mathbf{Q}_j-\mu_Y, g\rangle \leq r_0$ for at least $0.9k$ points. Therefore, for at least $0.8k$ points, both these two inequalities hold. Hence
\begin{eqnarray}
	\langle \mu_Y-\mathbf{c}, g\rangle &=& \langle \mathbf{Q}_j-\mathbf{c}, \mathbf{g}\rangle - \langle \mathbf{Q}_j-\mu_Y, \mathbf{g}\rangle\nonumber\\
	&\geq & d-r_0\nonumber\\
	&\geq & \frac{3}{4}\norm{\mu_Y-\mathbf{c}}-\frac{1}{4}\norm{\mu_Y-\mathbf{c}}\nonumber\\
	&=& \frac{1}{2}\norm{\mu_Y-\mathbf{c}}.
\end{eqnarray}
The proof of \eqref{eq:gerr} is complete.

Based on Lemma \ref{lem:distest}, we then move on to prove Lemma \ref{lem:iterate}.

\textbf{Proof of Lemma \ref{lem:iterate}.}  Recall that Lemma \ref{lem:iterate} requires that one of two conditions hold: (1) $\norm{\mathbf{c}_1-\mu_Y}\leq 4r_0$; (2) $t_c\geq 2\ln \frac{\norm{\mathbf{c}_1-\mu_Y}}{4r_0}/\ln \frac{256}{233}$. To begin with, we show that $\exists l\in \{1,\ldots, t_c+1\}$, $\norm{\mathbf{c}_l-\mu_Y}\leq 4r_0$. It automatically holds under condition (1). Therefore, we need to prove it with $\norm{\mathbf{c}_1-\mu_Y}>4r_0$, and $t_c\geq 2\ln \frac{\norm{\mathbf{c}_1-\mu_Y}}{4r_0}/\ln \frac{256}{233}$.

If $\norm{\mathbf{c}_l-\mu_Y}>4r_0$ for some $l$, then with $\eta = 1/4$,
\begin{eqnarray}
	\norm{\mathbf{c}_{l+1}-\mu_Y}^2 &=&\norm{\mathbf{c}_l+\eta d_l\mathbf{g}_l-\mu_Y}^2\nonumber\\
	&=& \norm{\mathbf{c}_l-\mu_Y}^2 + \eta^2 d_l^2 - 2\eta d_l \langle \mu_Y-\mathbf{c}_l, \mathbf{g}_l\rangle\nonumber\\
	&\overset{(a)}{\leq} & \norm{\mathbf{c}_l-\mu_Y}^2 + \eta^2 \frac{25}{16} \norm{\mathbf{c}_l-\mu_Y}^2 - 2\eta \frac{3}{4} \norm{\mathbf{c}_l-\mu_Y} \langle \mu_Y-\mathbf{c}_l, \mathbf{g}_l\rangle\nonumber\\
	&\overset{(b)}{\leq} &\norm{\mathbf{c}_l-\mu_Y}^2 +\frac{25}{16} \eta^2 \norm{\mathbf{c}_l-\mu_Y}^2 - 2\eta \frac{3}{4} \norm{\mathbf{c}_l-\mu_Y} \cdot \frac{1}{2}\norm{\mu_Y-\mathbf{c}_l}\nonumber\\
	&=& \left(1+\frac{25}{16}\eta^2-\frac{3}{4}\eta\right) \norm{\mathbf{c}_l-\mu_Y}^2 = \frac{233}{256}\norm{\mathbf{c}_l-\mu_Y}^2.
\end{eqnarray}
(a) comes from \eqref{eq:derr}, which ensures that $d_l\leq \norm{\mathbf{c}_l-\mu_Y}+r_0\leq 5\norm{\mathbf{c}_l-\mu_Y}/4$. (b) comes from \eqref{eq:gerr} in Lemma \ref{lem:distest}, which holds under the condition $\norm{\mathbf{c}_l-\mu_Y}>4r_0$. Therefore
\begin{eqnarray}
	\norm{\mathbf{c}_{l+1}-\mu_Y}\leq \sqrt{\frac{233}{256}} \norm{\mathbf{c}_l-\mu_Y}.
\end{eqnarray}
Therefore, with $\norm{\mathbf{c}_1-\mu_Y}>4r_0$, $\norm{\mathbf{c}_l-\mu_Y}$ decays exponentially with $l$, until $\norm{\mathbf{c}_l-\mu_Y}\leq 4r_0$. The first $l$ with $\norm{\mathbf{c}_l-\mu_Y}\leq 4r_0$ is at most at $\ln\frac{\norm{\mathbf{c}_1-\mu_Y}}{4r_0}/\ln \sqrt{\frac{256}{233}}$, which does not exceed $t_c$ in Algorithm \ref{alg:est_improve}. Therefore, there exists $l\in \{1,\ldots, t_c+1\}$, such that $\norm{\mathbf{c}_l-\mu_Y}\leq 4r_0$. Intuitively, this result indicates that in Algorithm \ref{alg:est_improve}, $\mathbf{c}_l$ will be close to $\mu_Y$.

However, the above argument does not ensure that the step $l$ with $\norm{\mathbf{c}_l-\mu_Y}\leq 4r_0$ will be selected to get the final estimate. Recall that Algorithm \ref{alg:est_improve} picks $l^*=\arg\min_l d_l$. Selecting $l$ with minimum $d_l$ is slightly different with picking $l$ with minimum $\norm{\mathbf{c}_l-\mu_Y}$. To solve this issue, we reuse \eqref{eq:derr}. Define
\begin{eqnarray}
	l_0=\arg\min_l \norm{\mathbf{c}_l-\mu_Y}.
\end{eqnarray}
Then
\begin{eqnarray}
	\norm{\mathbf{c}_{l^*}-\mu_Y} &\leq & d_{l^*}+r_0\nonumber\\
	&\leq & d_{l_0}+r_0\nonumber\\
	&\leq & \norm{\mathbf{c}_{l_0}-\mu_Y} + 2r_0\nonumber\\
	&\leq & 6r_0.
\end{eqnarray}

The proof of Lemma \ref{lem:iterate} is complete.
\section{Proof of Theorem \ref{thm:risknew}}\label{sec:risknew}
From \eqref{eq:g2new} and \eqref{eq:dpstep},
\begin{eqnarray}
	G^2\lesssim \frac{d}{n}+\frac{R^2dT}{n^2\epsilon^2}\ln \frac{1}{\delta}\left(\ln \frac{T}{\delta}+\sqrt{\ln \frac{T}{\delta}\ln d}\right) + d^p R^{2(1-p)}.
\end{eqnarray}

From Lemma \ref{lem:opt} and Lemma \ref{lem:bgnew}, 
\begin{eqnarray}
	&&\hspace{-1cm}\mathbb{E}[F(\hat{\mathbf{w}})]-F(\mathbf{w}^*)\nonumber\\
	 &\leq & \frac{L^2}{2\eta T} + LB + \eta(\lambda^2 L^2 + G^2)\nonumber\\
	&\overset{(a)}{=}& \frac{L^2}{2T} \sqrt{2T\lambda^2}+LB + \sqrt{\frac{1}{2T\lambda^2}}(L^2\lambda^2 + G^2)\nonumber\\
	&=& L^2 \lambda\sqrt{\frac{2}{T}} + LB + \frac{G^2}{\sqrt{2T\lambda^2}}\nonumber\\
	&\overset{(b)}{\lesssim} & \sqrt{\frac{1}{T}}+\sqrt{\frac{d}{n}}+d^\frac{p}{2}R^{1-p}+\frac{1}{\sqrt{T}}\left[\frac{d}{n}+\frac{R^2dT}{n^2\epsilon^2}\ln \frac{1}{\delta}\left(\ln \frac{T}{\delta}+\sqrt{\ln \frac{T}{\delta}\ln d}\right) + d^p R^{2(1-p)}\right]\nonumber\\
	&\overset{(c)}{\lesssim} & \frac{R\sqrt{d}}{n\epsilon}+\sqrt{\frac{d}{n}}+d^\frac{p}{2} R^{1-p}+\frac{R\sqrt{d}}{n\epsilon}\ln \frac{1}{\delta} \left(\ln \frac{T}{\delta}+\sqrt{\ln \frac{T}{\delta}\ln d}\right)\nonumber\\
	&\sim & \sqrt{\frac{d}{n}}+d^\frac{p}{2} R^{1-p}+\frac{R\sqrt{d}}{n\epsilon}\ln \frac{1}{\delta} \left(\ln \frac{T}{\delta}+\sqrt{\ln \frac{T}{\delta}\ln d}\right).
	\label{eq:riskbound}
\end{eqnarray}
For (a) and (c), recall that in the statement of Theorem \ref{thm:risknew}, the parameters are set to be $T=n^2\epsilon^2/(dR^2)$ and $\eta=1/\sqrt{2T\lambda^2}$. (b) uses Lemma \ref{lem:bgnew}.

To minimize \eqref{eq:riskbound}, let
\begin{eqnarray}
	R=\sqrt{d}\left(\frac{n\epsilon}{\sqrt{d}}\right)^\frac{1}{p}, 
\end{eqnarray}
then
\begin{eqnarray}
	T=\frac{n^2\epsilon^2}{dR^2}= \frac{1}{d}\left(\frac{n\epsilon}{\sqrt{d}}\right)^{2-\frac{2}{p}}.
\end{eqnarray}
\begin{eqnarray}
	\mathbb{E}[F(\hat{\mathbf{w}})]-F(\mathbf{w}^*)&\lesssim& \sqrt{\frac{d}{n}} + \sqrt{d}\left(\frac{\sqrt{d}}{n\epsilon}\right)^{1-\frac{1}{p}}\ln \frac{1}{\delta} \left(\ln \frac{T}{\delta}+\sqrt{\ln \frac{T}{\delta}\ln d}\right)\nonumber\\
	&\lesssim & \sqrt{\frac{d}{n}} + \sqrt{d}\left(\frac{\sqrt{d}}{n\epsilon}\right)^{1-\frac{1}{p}}\ln \frac{1}{\delta}(\ln n+\ln d).
\end{eqnarray}
\section{Technical Lemmas}\label{sec:lemmas}

Lemma \ref{lem:mgf} gives a bound of the moment generating function of a bounded random variable.
\begin{lem}\label{lem:mgf}
	(\cite{nguyen2023improved}, Lemma 2.2) Let $V$ be a one dimensional random variable such that $\mathbb{E}[V]=0$ and $|V|\leq R$ almost surely. Then for $0\leq \lambda\leq 1/R$,
	\begin{eqnarray}
		\mathbb{E}[e^{\lambda V}]\leq \exp\left(\frac{3}{4}\lambda^2 \mathbb{E}[V^2]\right).
	\end{eqnarray}
\end{lem}
\begin{proof}
	We provide a simplified proof following \cite{beygelzimer2011contextual}. Use the inequality $e^x\leq 1+x+(3/4)x^2$,
	\begin{eqnarray}
		\mathbb{E}[e^{\lambda V}]\leq \mathbb{E}\left[1+\lambda V + \frac{3}{4}\lambda^2 V^2\right]=  1+\frac{3}{4}\lambda^2 \mathbb{E}[V^2]\leq \exp\left(\frac{3}{4}\lambda^2 \mathbb{E}[V^2]\right).
	\end{eqnarray}
\end{proof}

Lemma \ref{lem:nondir} converts the moment bound at each direction to an overall moment bound on the norm of random vector $\mathbf{X}$.
\begin{lem}\label{lem:nondir}
	Under the condition $\mathbb{E}[|\langle \mathbf{u}, \mathbf{X}\rangle|^p]\leq M^p$, $\mathbb{E}[\norm{\mathbf{X}}^p]\leq d^\frac{p}{2} M^p$.
\end{lem}
\begin{proof}
	Note that for a set of orthonormal basis $\mathbf{e}_1,\ldots, \mathbf{e}_d$, for any vector $\mathbf{x}$,
	\begin{eqnarray}
		\norm{\mathbf{x}}^2 =\sum_{i=1}^d \langle\mathbf{e}_i, \mathbf{x}\rangle^2.
	\end{eqnarray}
	Then for $p\geq 2$, use H{\"o}lder inequality,
	\begin{eqnarray}
		\norm{\mathbf{x}}^p=\left(\sum_{i=1}^d \langle \mathbf{e}_i, \mathbf{x}\rangle^2\right)^\frac{p}{2}\leq d^{\frac{p}{2}-1} \sum_{i=1}^d |\langle \mathbf{e}_i, \mathbf{x}\rangle|^p.
	\end{eqnarray}
	Therefore
	\begin{eqnarray}
		\mathbb{E}[\norm{\mathbf{X}}^p]\leq d^{\frac{p}{2}-1} \sum_{i=1}^d \mathbb{E}[|\langle \mathbf{e}_i, \mathbf{x}\rangle|^p]\leq d^\frac{p}{2}M^p.
	\end{eqnarray}
\end{proof}

The following lemma is related to Rademacher random variables. These lemmas are common in statistical learning theory literatures \cite{mohri2018foundations}. In particular, it is also stated in Proposition 2 in \cite{lugosi2019sub}.
\begin{lem}\label{lem:rad}
	Let $U_1,\ldots, U_n$ be i.i.d random variables in $\mathbb{R}$. Let $\sigma_i$, $i=1,\ldots, n$ be $n$ independent Rademacher random variables, with $\text{P}(\sigma_i=1) = \text{P}(\sigma_i=-1)=1/2$. For a set of functions $\mathcal{F}$,
	\begin{eqnarray}
		\mathbb{E}\left[\underset{f\in \mathcal{F}}{\sup}\frac{1}{n}\sum_{i=1}^n (f(U_i)-\mathbb{E}[f(U_i)])\right] \leq 2\mathbb{E}\left[\underset{f\in \mathcal{F}}{\sup}\frac{1}{n}\sum_{i=1}^n \sigma_i f(U_i)\right].
		\label{eq:rad1}
	\end{eqnarray} 
	Moreover, if $\mathbb{E}[f(U_i)]=0$ for all $f\in \mathcal{F}$, then
	\begin{eqnarray}
		\mathbb{E}\left[\underset{f\in \mathcal{F}}{\sup}\frac{1}{n}\sum_{i=1}^n \sigma_i f(U_i)\right] \leq 2\mathbb{E}\left[\underset{f\in \mathcal{F}}{\sup}\frac{1}{n}\sum_{i=1}^n f(U_i)\right].
		\label{eq:rad2}
	\end{eqnarray}
\end{lem}

The following contraction lemma comes from \cite{ledoux2013probability}, which is stated again in Proposition 3 in \cite{lugosi2019sub}.
\begin{lem}\label{lem:contraction}
	Let $U_1,\ldots, U_n$ be i.i.d random variables in $\mathbb{R}$. Let $\sigma_i$, $i=1,\ldots, n$ be $n$ independent Rademacher random variables. Then
	\begin{eqnarray}
		\mathbb{E}\left[\underset{f\in \mathcal{F}}{\sup}\sum_{i=1}^n \sigma_i|f(U_i)|\right]\leq \mathbb{E}\left[\underset{f\in \mathcal{F}}{\sup}\frac{1}{n}\sum_{i=1}^n \sigma_i f(U_i)\right].
	\end{eqnarray}
\end{lem}
\end{document}